\documentclass[10pt]{amsart}

\usepackage{amsmath}
\usepackage{mathtools}
\usepackage{dsfont}
\usepackage{mathabx}

\usepackage[usenames, dvipsnames]{color}
\usepackage[colorlinks=true,
        raiselinks=true,
        linkcolor=MidnightBlue,
        citecolor=Brown,
        urlcolor=ForestGreen,
        pdfauthor={},
        pdftitle={},
        pdfkeywords={},
        pdfsubject={},
        plainpages=false]{hyperref}
\usepackage[T1]{fontenc}

\usepackage[norelsize, ruled, vlined]{algorithm2e}

\usepackage{enumitem}
\setlist{noitemsep, topsep=0cm}

\usepackage{tikz}
\usepackage{subfig}
\usepackage{url}
\usepackage{float}

\usetikzlibrary{fadings}
\usetikzlibrary{shapes.misc}
\tikzset{cross/.style={cross out, draw=black, minimum size=2*(#1-\pgflinewidth), inner sep=0pt, outer sep=0pt},cross/.default={1pt}}
\usetikzlibrary{shapes,arrows}
\usetikzlibrary{positioning}
\usetikzlibrary{plotmarks}

\allowdisplaybreaks

\newtheorem{theorem}{Theorem}[section]
\newtheorem{proposition}[theorem]{Proposition}
\newtheorem{lemma}[theorem]{Lemma}
\newtheorem{corollary}[theorem]{Corollary}

\theoremstyle{remark}
\newtheorem{remark}[theorem]{Remark}

\theoremstyle{definition}
\newtheorem{example}[theorem]{Example}

\newcommand{\norm}[1]{\left\lVert#1\right\rVert}
\newcommand{\pmat}[3]{\begin{pmatrix} #1 & #2 &\cdots & #3 \end{pmatrix}}
\newcommand{\pmatr}[1]{\begin{pmatrix}#1\end{pmatrix}}
\newcommand{\dict}[1]{d_{#1}}
\newcommand{\dictx}[1]{\delta_{#1}}
\newcommand{\inprod}[2]{\left\langle#1, #2\right\rangle}
\newcommand{\posdef}[1]{\mathbb{S}^{#1 \times #1}_{++}}
\newcommand{\possemdef}[1]{\mathbb{S}^{#1 \times #1}_{+}}

\DeclareSymbolFont{symbolsC}{U}{pxsyc}{m}{n}
\DeclareMathOperator{\trace}{tr}
\DeclareMathOperator{\image}{image}
\DeclareMathOperator{\Span}{span}
\DeclareMathOperator{\rank}{rank}
\DeclareMathOperator{\Ortho}{Ortho}
\DeclareMathOperator{\diag}{diag}
\DeclareMathOperator{\EE}{\mathsf{E}}
\DeclareMathOperator*{\minimize}{minimize}
\DeclareMathOperator*{\sbjto}{subject\,to}

\DeclarePairedDelimiterX\set[1]\lbrace\rbrace{\def\suchthat{\;\delimsize\vert\;}#1}

\newcommand{\Let}{\coloneqq}

\newcommand{\R}{\mathbb{R}}
\newcommand{\opt}{^\ast}
\newcommand{\transp}{^\top}

\newcommand{\xz}{x_{\mathrm{i}}}
\newcommand{\xf}{x_{\mathrm{f}}}
\newcommand{\reachindex}{K}
\newcommand{\reachmat}{\mathfrak{R}}

\renewcommand{\geq}{\geqslant}
\renewcommand{\ge}{\geqslant}

\renewcommand{\leq}{\leqslant}
\renewcommand{\mapsto}{\longmapsto}

\title{Optimal Dictionary for Least Squares Representation}
\author{Mohammed Rayyan Sheriff}
\author{Debasish Chatterjee}
\thanks{The authors are with Systems and Control Engineering, IIT Bombay, Mumbai 400076, India. Emails:\ (MRS) \texttt{mohammedrayyan@sc.iitb.ac.in}, (DC) \texttt{dchatter@iitb.ac.in}}

\begin{document}

\begin{abstract}
	Dictionaries are collections of vectors used for representations of random vectors in Euclidean spaces. Recent research on optimal dictionaries is focused on constructing dictionaries that offer sparse representations, i.e., \(\ell_0\)-optimal representations. Here we consider the problem of finding optimal dictionaries with which representations of samples of a random vector are optimal in an \(\ell_2\)-sense: optimality of representation is defined as attaining the minimal average $\ell_2$-norm of the coefficients used to represent the random vector. With the help of recent results on rank-\(1\) decompositions of symmetric positive semidefinite matrices, we provide an explicit description of $\ell_2$-optimal dictionaries as well as their algorithmic constructions in polynomial time.
\end{abstract}


\maketitle

\section{Introduction}
A \emph{dictionary} is a collection of vectors in a finite-dimensional vector space over $\mathbb{R}$, with which other vectors of the vector space are represented. A dictionary is a generalization of a basis: While the number of vectors in a basis is exactly equal to the dimension of the vector space, a dictionary may contain more elements. In this article we consider a problem of finding an optimal dictionary, where optimality is interpreted as the minimum expected average size of the coefficients required to represent a certain collection of vectors drawn from a given probability distribution.
		
We begin with a toy example to motivate the problems treated in this article. Let \( V \) be a random vector that attains values `close' to \(\pmatr{0 & 2}^\top\) with high probability; the situation is demonstrated in figure \ref{fig:comparison-of-dictionary}.

			\begin{figure}
			\centering
			\subfloat[]{
			\label{fig:part1}
			\begin{tikzpicture}
			\draw[->,dashed] (-1.75,0)--(2.5,0) node[right]{$x$};
\draw[->,dashed] (0,-1.75)--(0,3.5) node[above]{$y$};
			\coordinate (A) at (0,0);
			\coordinate (B) at (1.5,0.15);
			\coordinate (C) at (1.5,-0.15);
			 \coordinate (D) at (0,3);
			 \node[text width=2cm] at (1.5,3) {V};
			\draw [fill=blue] (A) circle (1pt) node [left] {};
			\draw [fill=blue] (B) circle (1pt) node [above right] {\( d_1 \)};
			\draw [fill=blue] (C) circle (1pt) node [below right] {\( d_2 \)};
			\draw [dotted](0,0) circle [radius=1.523];
			\draw [white,even odd rule,inner color=black,outer color = white](0,3) circle [radius=0.25];
			
			\draw [->, red, thick] (A) -- (B); 
			\draw [->, red, thick] (A) -- (C);

			\end{tikzpicture}
				}	
				\subfloat[]{
				\label{fig:part2}
			\begin{tikzpicture}
			\draw[->,dashed] (-1.75,0)--(2.5,0) node[right]{$x$};
\draw[->,dashed] (0,-1.75)--(0,3.5) node[above]{$y$};
			\coordinate (A) at (0,0);
			\coordinate (B) at (0.15, 1.5);
			\coordinate (C) at (-0.15, 1.5);
			 \coordinate (D) at (0,3);
			  \node[text width=2cm] at (1.5,3) {V};
			\draw [fill=blue] (A) circle (1pt) node [left] {};
			\draw [fill=blue] (B) circle (1pt) node [above right] {\( d_1\opt \)};
			\draw [fill=blue] (C) circle (1pt) node [above left] {\( d_2\opt \)};
			\draw [dotted](0,0) circle [radius=1.523];
			\draw [white,even odd rule,inner color=black,outer color = white](0,3) circle [radius=0.25];
			
			\draw [->, red, thick] (A) -- (B); 
			\draw [->, red, thick] (A) -- (C);

			\end{tikzpicture}
				}
				
	\caption{Comparison of two dictionaries. }	
	\label{fig:comparison-of-dictionary}	
		\end{figure}
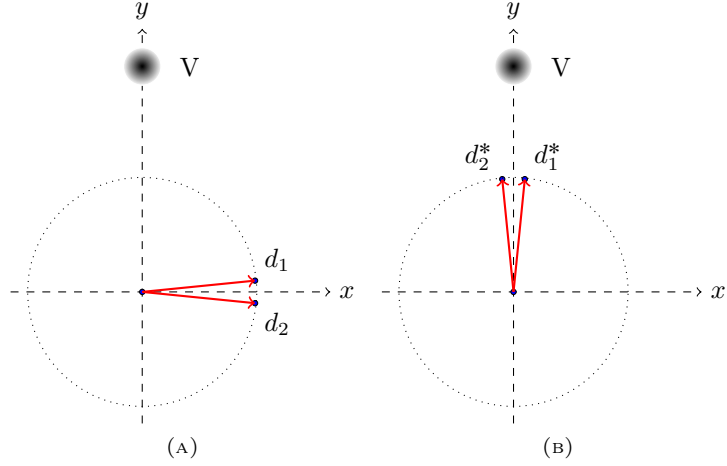
Suppose that our dictionary consists of the vectors $ d_1 = \pmatr{1 & -\epsilon}^\top $ and $ d_2 = \pmatr{1  &\epsilon}^\top $ in $\mathbb{R}^2$, with a small positive value of $\epsilon$. Since we must represent \( V \) using $d_1$ and $d_2$, the corrresponding coefficients \( \alpha_1 \) and \( \alpha_2 \) must be such that $\alpha_1 \pmatr{1 & \epsilon}^\top + \alpha_2 \pmatr{1 & -\epsilon}^\top = V \approx \pmatr{0 & 2}^\top $. A quick calculation shows that the magnitudes of the coefficients \( \alpha_1 \) and \( \alpha_2 \) should then be approximately equal to \( 1/(\epsilon) \) with high probability. To wit, the magnitudes of these coefficients are large for small values of $\epsilon$. It is therefore more appropriate in this situation to consider a dictionary consisting of vectors $ d_1\opt = \pmatr{\epsilon & 1}^\top $ and $ d_2\opt = \pmatr{-\epsilon & 1}^\top $ to represent the samples of \( V \), in which case, the magnitudes of the coefficients of the representations are closer to $1$ with high probability. The latter values are comparatively far smaller compared to the values close to $1/(\epsilon)$ obtained with the preceding dictionary. This simple example shows that given some statistical information about the random vectors to be represented, the question of designing a dictionary that minimizes the average cost of representation can be better addressed. 

Let us now turn to a situation in which considering the average cost of representations is natural. Our motivation comes from a control theoretic ideas perspective. Consider a linear time-invariant control system modeled by the recursion
\begin{equation}
\label{linear system}
x(t+1) = A x(t) + B u(t), \quad t = 0, 1, \ldots,
\end{equation}
where the `system matrix' \(A\in\R^{n\times n}\) and the `control matrix' \(B\in\R^{n\times m}\) are given, with the initial boundary condition \(x(0) = \xz\in\R^n\) fixed. For an arbitrarily selected \(\xf\in\R^n\), consider the standard \emph{reachability problem} for \eqref{linear system}, that is:

\begin{equation}
\label{eq:lti-reachability-problem}
\begin{aligned}
& \text{If possible, find a sequence \( (u(t))_{t}\subset\R^m \) of control vectors }\\
& \text{that steer the system states to \( \xf \)}.
\end{aligned}
\end{equation}
A necessary and sufficient condition for such a sequence to exist for every pair \((\xz, \xf)\) is that the rank of the matrix \( \reachmat_{\reachindex} (A, B) \Let \pmatr{B & AB & \cdots & A^{n-1}B} \) is equal to \( n \), which we impose for the moment. Letting \(\reachindex \Let \min \; \set[\big]{k\ge 0 \suchthat \rank \left( \reachmat_{\reachindex} (A, B) \right) = n}\) denote the `reachability index' of \eqref{linear system}, we see at once that the control vectors $(u(t))_{t=0}^{\reachindex-1}$ needed to execute the transfer of the states of \eqref{linear system} from $\xz$ to $\xf$ must be a solution to the linear equation
\[
	\xf - A^\reachindex \xz = \sum_{t=0}^{\reachindex-1} A^t B u(t) = \reachmat_{\reachindex} (A, B) \pmatr{u(\reachindex-1)\\\vdots \\ u(1)\\ u(0)}.
\]
It is now natural to consider the `control cost' of transferring \(\xz\) to \(\xf\), for which, a natural candidate is the associated \(\ell_2\) performance index $ \sum_{i = 0}^{\reachindex - 1} \norm{u(t)}^2 $. Since in practice, the \( \ell_2 \) performance index is analogous to the amount of energy spent to control the system, its practical importance can hardly be overstated in the context of control. Let us list three examples: 
\begin{itemize}[label=\(\circ\), leftmargin=*]
\item In attitude control/orientation problems of space vehicles, one must execute most of the rapid manoeuvers using the energy from the limited amount of fuel on board, or with the energy available from on-board batteries; minimizing the energy expenditure, therefore, is crucial.
\item In controlled automated mobile robots (e.g., automated cars) designed to reach a given location within a certain time, reduction of energy consumption leads directly to reduction in fuel consumed.
\item In control of electronic systems such as power electronic drives, the associated \( \ell_2 \) performance index involves information of the amount of power drawn from the electricity grid to control the system, leading directly to minimization of power consumption and thereby heating.
\end{itemize}
Minimization of control effort has been an integral part of control theory, and is generally studied under the class of Linear Quadratic problems; see, e.g., \cite{bertsekas1995dynamic}, \cite{anderson2007optimal}, \cite{clarke2013functional}, \cite{liberzon2012calculus}, or any standard book on optimal control. It is evident that the task of designing control systems that require minimum control energy for their typical manoeuvres is of great importance.
It is a standard practice to study the reachability problem \eqref{eq:lti-reachability-problem}, for $\xz = 0$ and $\xf$ on the unit sphere; due to linearity of \eqref{linear system}, this special case provides sufficient insight into the general case. Let us consider the following optimal control problem: 
\begin{equation}
\label{eq:optimal-control-problem}
\begin{aligned}
& \minimize_{(u(t))_t}	&& \EE \biggl[ \sum_{t = 0}^{\reachindex - 1} \norm{u(t)}^2 \biggr]\\
		& \sbjto							&&
		\begin{cases}
			x(t + 1) = A x(t) + B u(t) \quad \text{for all $ t = 0,\ldots,K - 1 $},\\
			x(0) = 0,\\
			x(K) = \hat x \text{ distributed according to \(\mu\)},
		\end{cases}
\end{aligned}
\end{equation}
where \(\mu\) is a probability distribution on \(\R^d\). It is known that if \( \hat x \) is uniformly distributed over the unit sphere, then the optimal control problem \eqref{eq:optimal-control-problem} admits an unique optimal solution and the optimum value is proportional to $\trace\bigl(W^{-1}_{A, B}\bigr)$, where $W_{A, B} \coloneqq \reachmat_{\reachindex}(A, B)\reachmat_{\reachindex}(A, B)\transp$ is the \emph{controllability grammian} of the system; for details see, e.g., \cite{ref:MulWeb-72} and \cite{pasqualetti2014controllability}. It can be readily shown that if \( \Sigma \coloneqq \EE [ \hat x \hat x^{\top} ] \) is well defined, then the optimum value of \eqref{eq:optimal-control-problem} is equal to \( \trace\bigl( \Sigma W^{-1}_{A, B}\bigr) \). Evidently, for a given distribution of \( \hat x \), different linear systems \eqref{linear system} --- described completely by the pair \((A, B)\) --- incur different optimum values \( \trace\bigl( \Sigma W^{-1}_{A, B}\bigr) \) of \eqref{eq:optimal-control-problem}.

Against the above backdrop, consider the question of \emph{designing} the linear control system \eqref{linear system} such that the value of \eqref{eq:optimal-control-problem} is as low as possible. Since most control problems involve designing control sequences to execute a class of desired manoeuvres, for a given distribution of \( \hat x \) it is then natural to design the linear systems in order to minimize the optimum value \( \trace\bigl( \Sigma W^{-1}_{A, B}\bigr) \) of the optimal control problem \eqref{eq:optimal-control-problem}. In this case, the system design problem is similar to the one of finding an \(\ell_2\)-optimal dictionary as described above: here the matrices $A$ and $B$ are to be designed, within a feasible region, such that the column vectors constituting the matrix $\reachmat_\reachindex(A, B)$ lead to minimal expected average cost of reachability, i.e., minimal value of \eqref{eq:optimal-control-problem}. Such problems routinely arise in networked control, where the pair $(A, B)$ is a function of the constituent systems and the connectivity of the network. From an operational standpoint, it is good for a networked system to have its components connected in a way such that the resulting system incurs small expected average state transfer costs. Indeed, control systems are typically designed \cite{ref:MulWeb-72} by optimizing a \emph{figure of merit} / \emph{measure of quality} / \emph{measure of controllability}; in particular, networked control systems are designed in \cite{pasqualetti2014controllability} using a measure of quality defined there. Based on this work on \( \ell_2 \)-optimal dictionaries, we have proposed a novel measure of quality in \cite{sheriff2017frame}, and further developments for algorithmic synthesis of large-scale control systems will be reported elsewhere. Besides these applications in control theory and practice, one of the key objective of our work here is to investigate and understand the physical nature of the \( \ell_2 \)-optimal dictionaries independent of their connection with control theory. Such a study will shed light on other control theoretic properties of observability and estimation.

There has been significant recent research into finding optimal dictionaries, briefly outlined in \cite{tovsic2011dictionary}; current research centers around the development of learning algorithms for finding optimal dictionaries. Much of the thrust is on arriving at dictionaries that offer sparse representations of sample vectors. One of the first learning algorithms to develop a dictionary that offers sparse representation of images was given in \cite{olshausen1997sparse}. Since then many learning algorithms have been developed to obtain dictionaries that offer sparse representation along with other special properties such as online computation capability \cite{mairal2009online}, better classification property \cite{mairal2009supervised,yang2011fisher}, better adaptive  properties \cite{skretting2010recursive}; several other algorithms are given in \cite{kreutz2003dictionary,yaghoobi2009dictionary,mallat1993matching}.

The problem addressed in this article differs from the mainstream research of finding dictionaries offering sparse (\(\ell_0\)-optimal) representations in the sense that our objective is to find dictionaries that give minimum average $\ell_2$-norm of the coefficient vector used for representation. Intuitively, optimization of the $\ell_2$-norm of the representation vector tends to `distribute' the information of the data being represented among all components of the representation vector; this makes the representation robust to accidental changes in the coefficients.
\begin{itemize}[label = \( \circ \), leftmargin = *]
	\item An advantage of considering the $\ell_2$-cost is that it involves a norm arising from an inner product; consequently, it comes with a rich set of properties associated with it. These properties are crucially employed in this article to modify the intrinsically non-convex problem of finding an \(\ell_2\)-optimal dictionary into an \emph{equivalent} convex optimization problem,\footnote{By equivalence of two optimization problems we mean that an optimal solution to either of the problems can be obtained from an optimal solution to the other problem.} allowing us to compute an optimal dictionary in \emph{polynomial time} and arrive at \emph{analytical expressions of the optimal costs}. We provide these algorithms in Sections \ref{sec:proof-of-DL-theorem} and \ref{sec:proof-of-general-theorem}.

	\item One more advantage of considering optimization in the $\ell_2$-sense is related to the fact that the $\ell_2$-cost involves the natural notion of \emph{energy} which is extremely important in practice, especially in control theoretic applications.

	\item The results presented here also add to the recent developments in the advantages of representing signals/vectors using tight frames for finite-dimensional Hilbert spaces.
\end{itemize}

This article unveils as follows: In Section \ref{The Dictionary Learning problem} we formally introduce our problem of finding an optimal dictionary which offers least square representation. Section \ref{The Dictionary Learning problem} is the heart of this article, where we solve the problem of finding an \(\ell_2\)-optimal dictionary, and arrive at an explicit solution. Algorithms to construct \(\ell_2\)-optimal dictionaries are given in Section \ref{Proofs}, where we present the proofs of our main results. The case of representing random vectors distributed uniformly on the unit sphere is treated in Subsection \ref{uniform section}; we demonstrate that the \(\ell_2\)-optimal dictionaries in this case are \emph{finite tight frames}. The intermediate Section \ref{matrices} contains results related to rank-\(1\) decomposition of positive semidefinite matrices; these constitute essential tools for the solutions of our main results. We conclude in Section \ref{s:conclusion} with a summary of this work and future directions.

\subsection*{Notations}
We employ standard notations in this article. As usual, \(\norm{\cdot}\) is the standard Euclidean norm. The \(n\times n\) identity and \(m\times n\) zero matrices are denoted by \(I_n\) and \(O_{m\times n}\), respectively. For a matrix \(M\) we let \(\trace(M)\) and \(M^+\) denote its trace and Moore-Penrose pseudo-inverse, respectively. The set of \(n\times n\) symmetric and positive (semi-)definite matrices with real entries is denoted by \(\posdef{n}\) (\(\possemdef{n}\)), and the set of \(n\times n\) symmetric matrices with real entries is denoted by \(\mathbb{S}^{n\times n}\). For a Borel probability measure \(\mu\) defined on \(\R^n\), we let \(\EE_\mu[\cdot]\) denote the corresponding mathematical expectation. The image of a map \(f\) is written as \(\image(f)\). The gradient of a continuously differentiable function \(f\) is denoted by \(\nabla f\). For finite ordered sets \(A\) and \(B\), we let \(A\uplus B\) denote the ordered set consisting of the elements (in their order) of \(A\) followed by the elements (in their order) of \(B\); for instance, if \(A = (1, 2)\) and \(B = (-5, -7)\), then \(A\uplus B = (1, 2, -5, -7)\). Suppose that $A$ and $B$ are two ordered sets such that $B \subset A$ as sets, then $A \setminus B$ is the ordered sub-collection in $A$ after deleting the elements of the set $B$. Finally, given an ordered collection of vectors \((x_i)_{i=1}^n\) in \(\R^\nu \) with \( \nu \ge n\) and equipped with the standard inner product, \(\Ortho\bigl((x_i)_{i=1}^n\bigr)\) gives the result of Gram-Schmidt orthonormalization of the collection \((x_i)_{i=1}^n\) considered in the order in which they appear i.e., \( x_1, x_2 ,\ldots , x_n\).

\section{The $\ell_2$-optimal dictionary problem and its solution}
\label{The Dictionary Learning problem}
Let $V$ denote an \(\R^n\)-valued random vector defined on some probability space, and having distribution (i.e., Borel probability measure,) $\mu$. We assume that \(V\) has finite variance. Let $R_V$ denote the support of $\mu$,\footnote{Recall \cite[Theorem 2.1, Definition 2.1, pp.\ 27-28]{ref:Par-05} that the support of \(\mu\) is the set of points \(z\in\R^n\) such that the \(\mu\)-measure of every open neighbourhood of \(z\) is positive.} and let $X_V$ be the smallest subspace of \(\R^n\) containing $R_V$. Our goal is to represent the instances/samples of $V$ with the help of a \emph{dictionary} of vectors:
\[
	D_K \Let \set[\big]{\dict{i} \in \R^n \suchthat \norm{\dict{i}} = 1 \text{ for} \; i = 1,\ldots, K}\quad \text{ with a given \(K \geq n\)},
\]
in an optimal fashion. A \emph{representation} of an instance \(v\) of the random vector $V$ is given by the coefficient vector  $\alpha = (\alpha_1 \; \ldots \; \alpha_K)^\top $, such that
\begin{equation}
\label{21}
v = \sum_{i = 1}^K \alpha_i \dict{i}.
\end{equation}
A \emph{reconstruction} of the sample $v$ from the representation \(\alpha\) is carried out by taking the linear combination $\sum_{i = 1}^K \alpha_i \dict{i}$. We define the \emph{cost} associated with representing $v$ in terms of the coefficient vector $\alpha$ as $\sum_{i = 1}^K \alpha_i^2$. Since the dictionary vectors $\{ \dict{i} \}_{i = 1}^K$ must be able to represent any sample of $V$, the property that $\Span \{ \dict{i} \}_{i = 1}^K \supset R_V$ is essential. A dictionary $D_K = \{ \dict{i} \}_{i = 1}^K \subset\R^n$ is said to be \emph{feasible} if $\Span \{ \dict{i} \}_{i = 1}^K \supset R_V$. We denote by $\mathcal{D}_K$ the set of all feasible dictionaries.

For a feasible dictionary \(D_K = \{\dict{i}\}_{i=1}^K\), with \( m \Let \dim \left( \Span \{ \dict{i} \}_{i = 1}^K \right) \), and for any \(v\in R_V\), the linear equation \eqref{21} is satisfied by infinitely many values of $\alpha$ whenever $K > m$. In fact, the solution set of \eqref{21} constitutes a $(K - m)$-dimensional affine subspace of $\mathbb{R}^K$. Therefore, in order to represent a given $v$ uniquely, one must define a mechanism of selecting a particular point from this affine subspace, thus making the coefficient vector $\alpha = (\alpha_1 \; \ldots \; \alpha_K)^\top$ a function of $v$. Let \(f\) denote such a function; to wit, $f(v) \Let \alpha$ is the coefficient vector used to represent the sample $v$. We call such a map $R_V \ni v  \longmapsto f(v) \in \mathbb{R}^K$ a \emph{scheme of representation}. Representation of samples of the random vector \(V\) using a dictionary $D_K$ and a scheme $f$ is said to be \emph{proper} if any vector $v \in R_V$ can be uniquely represented and then exactly reconstructed back. It is clear that for proper representation of $V$ with a dictionary \(D_K\) consisting of vectors $\{ \dict{i} \}_{i = 1}^K$, the mapping  $R_V \ni v  \longmapsto f(v) \in \mathbb{R}^K$ should be an injection that satisfies
\begin{equation}
\label{injectivity}
V = \pmat{ \dict{1} }{ \dict{2} }{ \dict{K} }f(V) \quad \text{$\mu$-almost surely.}
\end{equation}
A scheme $f$ of representation is said to be \emph{feasible} if for some feasible dictionary $D_K \Let \{ \dict{i} \}_{i = 1}^K \in \mathcal{D}_K$ the equality $\pmat{ \dict{1} }{ \dict{2} }{ \dict{K} }f(V) = V$ is satisfied almost surely. We denote by $\mathcal{F}$ the set of all feasible schemes of representation. 

Given a scheme $f$ of representation, the (random) cost associated with representing $V$ is given by $\norm{f(V)}^2$. The problem of finding an \(\ell_2\)-optimal dictionary can now be posed as:
\begin{quote}
	Find a pair consisting of a dictionary $D_K\opt \in \mathcal{D_K}$ and a feasible scheme \(f\opt\) of representation such that the average cost $\EE_\mu \bigl[ \norm{f^*(V)}^2 \bigr]$ of representation is minimal.
\end{quote}
Here the subscript $\mu$ indicates the distribution of random vector $V$ with respect to which the expectation is evaluated. In other words, we have the following optimization problem:
\begin{equation}
	\label{DL problem general}
	\begin{aligned}
		& \minimize_{D_K, f}	&& \EE_\mu\bigl[\norm{f(V)}^2\bigr]\\
		& \sbjto				&& 
		\begin{cases}
			D_K \in \mathcal{D}_K, \\
			f \in \mathcal{F}.
		\end{cases}
	\end{aligned}
\end{equation}

The problem given in \eqref{DL problem general} will be referred to as the \emph{$\ell_2$-optimal dictionary} problem. It should be noted that the $\ell_2$-optimal dictionary problem is non-convex due to the constraint that the dictionary vectors $ \{ \dict{i} \}_{i = 1}^K $ of a feasible dictionary must be of unit length. Even if we change this constraint to $ \left\{ \norm{ \dict{i} } \leq 1 \right\}$ from $ \left\{ \norm{ \dict{i} } = 1 \right\} $, which makes the feasible region of dictionary vectors convex, the set of feasible schemes of representation is not known to be a convex set a priori.

In this article we solve the $\ell_2$-optimal dictionary problem given in \eqref{DL problem general} in two steps:
\begin{enumerate}[label=(Step \Roman*), leftmargin=*, align=left, widest=II]
	\item We let $X_V = \R^n$.
	\item We let \(X_V\) be any proper nontrivial subspace of \(\R^n\).\footnote{The trivial case of \(X_V = \{0\}\) is discarded because then there is nothing to prove; we therefore limit ourselves to `nontrivial' subspaces of \(\R^n\).}
\end{enumerate}
The remainder of this section is devoted to describing Steps I and II by exposing our main results, followed by discussions, a numerical example, and a treatment of the important case of the uniform distribution on the unit sphere of \(\R^n\).

\subsection{Step I: \(X_V = \R^n\)}

If $X_V = \mathbb{R}^n$, a dictionary of vectors $D_K = \{ \dict{i} \}_{i = 1}^K \subset \mathbb{R}^n$ is feasible if and only if $\norm{ \dict{i} } = 1$ for all $i = 1,\ldots, K$, and $\Span\{ \dict{i} \}_{i = 1}^K = \mathbb{R}^n$. Thus, the $\ell_2$-optimization problem \eqref{DL problem general} reduces to:
\begin{equation} 
	\label{DL problem}
	\begin{aligned}
		& \minimize_{\{d_i\}_{i=1}^K, f}	&& \EE_\mu\bigl[\norm{f(V)}^2\bigr]\\
		& \sbjto							&&
		\begin{cases}
			\norm{ \dict{i} } = 1 \text{ for all $i = 1,\ldots,K,$}\\
			\Span\{ \dict{i} \}_{i = 1}^K = \mathbb{R}^n, \\
			\pmatr{\dict{1} & \dict{2} & \cdots & \dict{K}}f(V) = V \text{\ \ \(\mu\)-almost surely}.
		\end{cases}
	\end{aligned}
\end{equation}

Let $\Sigma_V \Let \EE_\mu[VV\transp]$. We claim that $\Sigma_V$ is positive definite. Indeed, if not, then there exists a nonzero vector $x \in \mathbb{R}^n$ such that $x^\top V = 0$ almost surely, which contradicts the assumption that $X_V = \mathbb{R}^n$. 

Existence and characterization of the optimal solutions to \eqref{DL problem}   is done by the following:
\begin{theorem}
\label{DL theorem}
Consider the optimization problem \eqref{DL problem}, and let \(\Sigma_V \Let \EE_\mu\bigl[V V\transp\bigr]\).
\begin{itemize}[label=\(\circ\), leftmargin=*]
	\item \eqref{DL problem} admits an optimal solution.
	\item The optimal value corresponding to \eqref{DL problem} is $\dfrac{\bigl(\trace(\Sigma_V^{1/2})\bigr)^2}{K}$.
	\item Optimal solutions of \eqref{DL problem} are characterized by:
		\begin{itemize}[label=\(\triangleright\), leftmargin=*]
			\item a dictionary $D_K^* = \{ \dict{i}^*\}_{i = 1}^K$ that is feasible for \eqref{DL problem} and that satisfies
				\begin{equation}
					\sum_{i = 1}^K  \dict{i}^*{ \dict{i}^*}^\top = M^* \Let \; \frac{K}{\trace\bigl( \Sigma_V^{1/2} \bigr)}\; \Sigma_V^{1/2},
				\end{equation}
				and
			\item a scheme $f_{D_K^*}^*(v) \Let \pmat{ \dict{1}^*}{ \dict{2}^*}{ \dict{K}^*}^+ v$.
		\end{itemize}
\end{itemize}
Moreover, all optimal dictionary-scheme pairs can be obtained via the procedure described in Algorithm \ref{procedure 2} on p.\ \pageref{procedure 2}.
\end{theorem}

\subsection{Step II: \(X_V\) is a strict nontrivial subspace of \(\R^n\)}

Let \(X_V\) be any proper nontrivial subspace of \(\R^n\). In this situation it is reasonable to expect that no optimal dictionary that solves \eqref{DL problem general} contains elements that do not belong to \(X_V\). That this indeed happens is the assertion of the following Lemma, whose proof is provided in Section \ref{Proofs}:
\begin{lemma}
	\label{lemma 2}
	Optimal solutions, if any exists, of problem \eqref{DL problem general} are such that the optimal dictionary vectors $\{\dict{i}\opt\}_{i = 1}^K$ satisfy $\dict{i}\opt \in X_V$ for all $i = 1,\ldots,K$.
\end{lemma}
Lemma \ref{lemma 2} guarantees that if the problem \eqref{DL problem general} admits a solution, then the corresponding optimal dictionary vectors must be elements of $X_V$. This means that it is enough to optimize over dictionaries with their elements in $X_V$ instead of the whole of $\mathbb{R}^n$. Therefore, the constraint $\Span\{\dict{i}\}_{i = 1}^K \supset R_V$ can be equivalently stated as $\Span\{\dict{i}\}_{i = 1}^K = X_V$.

Let the dimension of $X_V$ be $m$ with $m < n$, and let $\mathcal{B} = \{b_i\}_{i = 1}^m$ be a basis for $X_V$. It should be noted that $X_V = \image(\Sigma_V)$, and therefore, a basis of $X_V$ can be obtained by computing a basis of the subspace $\image(\Sigma_V)$. An example of such a basis of $X_V$ is the collection of unit eigenvectors of $\Sigma_V$ corresponding to its non-zero eigenvalues.

Fix a basis \(\mathcal B = \{b_i\}_{i=1}^m\) of \(X_V\). Let $B$ be a matrix containing the vectors $\{b_i\}_{i=1}^m$ as its columns:
\[
	B \Let \pmatr{b_1 & b_2 & \cdots & b_m}.
\]
If $\dictx{i}$ is the representation of the dictionary vector $\dict{i}$ in the basis $\mathcal{B}$, i.e., $\dict{i} = B \dictx{i}$, then the constraints on the family $\{\dict{i}\}_{i=1}^K$ get transformed to the following ones on $\{\dictx{i}\}_{i=1}^K$:
\begin{itemize}[label=\(\circ\), leftmargin=*]
\item $\norm{\dict{i}}^2 = 1\quad \Rightarrow\quad \dictx{i}^\top \bigl(B^\top B\bigr) \dictx{i} = 1$, and
\item $\Span\{\dict{i}\}_{i = 1}^K \supset R_V\quad\Rightarrow\quad\Span\{\dict{i}\}_{i = 1}^K = X_V\quad\Rightarrow\quad\Span\{\dictx{i}\}_{i = 1}^K = \mathbb{R}^m$.
\end{itemize}

We define the random vector
\[
	V_X \Let \bigl((B^\top B)^{-1}B^\top\bigr) V.
\]
Then $V_X$ is an $\mathbb{R}^m$ valued random vector which is the representation of random vector $V$ in the basis $\mathcal{B}$. For every scheme $f$ that is feasible for \eqref{DL problem general}, let us define an associated scheme for representing samples of the random vector $V_X$ by
\[
	\R^m\ni v\longmapsto f_X(v) \Let f(B v)\in\R^K.
\]
The conditions on feasibility of $f$ in \eqref{DL problem general} imply that the scheme $f_X$ is feasible if for a feasible dictionary of vectors $\{\dictx{i}\}_{i = 1}^K$,
\[
	\pmat{\dictx{1}}{\dictx{2}}{\dictx{K}} f_X(V_X) = V_X \quad\text{\(\mu\)-almost surely.}
\]
In other words, in contrast to the problem \eqref{DL problem general}, where the optimization is carried out over vectors in $\mathbb{R}^n$, we can equivalently consider the same problem in $\mathbb{R}^m$, but with the following modified constraints: 
\begin{equation}\label{DL problem general 2}
\begin{aligned}
	& \minimize_{\{\dictx{i}\}_{i=1}^K, f_X}	&& \EE_\mu\bigl[\norm{f_X(V_X)}^2\bigr]\\
	& \sbjto									&&
	\begin{cases}
		\dictx{i}^\top \bigl( B^\top B \bigr) \dictx{i} = 1 \text{ for all $i = 1,\ldots,K,$}\\
		\Span\{\dictx{i}\}_{i = 1}^K = \mathbb{R}^m,\\
		\pmatr{\dictx{1} & \dictx{2} & \cdots & \dictx{K}} f_X(V_X) = V_X \text{\ \ \(\mu\)-almost surely}.
	\end{cases}
\end{aligned}
\end{equation}

In relation to the problem \eqref{DL problem general 2} let us define the following quantities
\begin{equation}
\label{H}
\left\{
\begin{aligned}
\Sigma_V	&\Let \EE_{\mu}[ V V^\top] \\
\Sigma		&\Let (B^\top B)^{-1/2} \bigl(B^\top \Sigma_V B \bigr) (B^\top B)^{-1/2} \\
H^*			&\Let \frac{K}{\trace\bigl( \Sigma^{1/2} \bigr)}\bigl((B^\top B)^{-1/2}\Sigma^{1/2} (B^\top B)^{-1/2}\bigr).
\end{aligned}
\right.
\end{equation}
Since the support of $V_X$ is \(m\)-dimensional, we conclude from previous discussion that $\Sigma_{V_X} \Let \EE_{\mu} \bigl[ V_X V_X^\top \bigr]$ is positive definite. Since $\Sigma = (B^\top B)^{1/2} \Sigma_{V_X} (B^\top B)^{1/2}$, it follows that $\Sigma$ is positive definite, which in turn implies that $H^*$ is positive definite.

To summarize, an $\ell_2$-optimal dictionary-scheme pair that solves the optimization problem \eqref{DL problem general} is equivalently obtained from an optimal solution of the problem \eqref{DL problem general 2}, and is characterized by the following:
\begin{theorem}
\label{DL general theorem}
	Consider the optimization problem \eqref{DL problem general 2}. 
	\begin{itemize}[label=\(\circ\), leftmargin=*]
		\item \eqref{DL problem general 2} admits an optimal solution.
		\item The optimal value corresponding to \eqref{DL problem general 2} is $\dfrac{\bigl(\trace(\Sigma^{1/2})\bigr)^2}{K}$.
		\item Optimal solutions of \eqref{DL problem general 2} are characterized by:
			\begin{itemize}[label=\(\triangleright\), leftmargin=*]
				\item a dictionary $D_K^* = \{ \dictx{i}^*\}_{i = 1}^K$ that is feasible for \eqref{DL problem general 2} and that satisfies
					\begin{equation}
						\sum_{i = 1}^K  \dictx{i}^*{ \dictx{i}^*}^\top = H^*,
					\end{equation}
					and
				\item a scheme $f_X^*(u) \Let \pmat{ \dictx{1}^*}{ \dictx{2}^*}{ \dictx{K}^*}^+ u$.
			\end{itemize}
	\end{itemize}
	Consequently, an optimal solution of the $\ell_2$-optimal dictionary problem \eqref{DL problem general} consisting of an $\ell_2$-optimal dictionary-scheme pair is given by 
	\begin{itemize}[label=\(\circ\), leftmargin=*]
		\item A collection of vectors $\{\dict{i}^*\}_{i = 1}^K$ defined as \(\dict{i}^* \Let B \dictx{i}^*\) for \(i = 1,2,\ldots,K\), and
		\item the scheme $f^*(v) \Let \pmat{\dict{1}^*}{\dict{2}^*}{\dict{K}^*}^+ v$.
	\end{itemize}
	Moreover, all optimal dictionary-scheme pairs can be obtained via the procedure given in Algorithm \ref{algo:general_l2_opt_algo} on p.\ \pageref{algo:general_l2_opt_algo}.
\end{theorem}

\subsection{Discussion and a numerical example}
\begin{remark}
	\label{r:linearity of optimal schemes}
	The problem \eqref{DL problem general} does not a priori hypothesize an affine/linear structure of candidate schemes. The fact that linear schemes are optimal in \eqref{DL problem general} is one of the crucial assertions of both Theorem \ref{DL theorem} and Theorem \ref{DL general theorem}.
\end{remark}

\begin{remark}
	Algorithmic computation of an $\ell_2$-optimal dictionary relies on the second moment $\Sigma_V$ of the random vector $V$. Complete knowledge of the distribution \(\mu\) is, therefore, unnecessary. This is an advantage since in practical situations, learning/estimating $\Sigma_V$ from data is comparatively less demanding than getting a description of the distribution $\mu$ itself.
\end{remark}

\begin{remark}
\label{remark:robustness}
Let \( M \in \possemdef{n} \) be such that \( \image(M) = X_V \), let \( \mathcal{B} = \{ b_i \}_{i = 1}^m \) be a basis for \( X_V \) evaluated as a basis for \( \image(M) \). Let 
\begin{align*}
B         & \Let \pmat{b_1}{b_2}{b_m} \\
\Sigma(M) & \Let (B^\top B)^{-1/2} \bigl(B^\top M B \bigr) (B^\top B)^{-1/2} \\
H(M)		  & \Let \frac{K}{\trace \left( \big( \Sigma(M) \big)^{1/2} \right)}\bigl((B^\top B)^{-1/2} \big( \Sigma(M) \big)^{1/2} (B^\top B)^{-1/2}\bigr).
\end{align*}
Suppose that \( \{ \dict{i} \}_{i = 1}^K \) and \( f(\cdot) \) are the dictionary and the scheme obtained using the procedure given in Algorithm \ref{algo:general_l2_opt_algo} using \( M \) and \( K \) as inputs. By simplifying the pseudo-inverse \( \pmat{\dict{1}}{\dict{2}}{\dict{K}}^+ \) in \( f(\cdot) \), the average cost \( J(M) \) of representing \( V \) using the scheme \( f(\cdot) \) turns out to be
\begin{equation}
\label{eq:cost-robustness}
\begin{aligned}
J(M) & = \mathsf{E_{\mu}} \left[ V\transp B (B\transp B)^{-1} \big( H(M) \big)^{-1} (B\transp B)^{-1} B\transp V \right] \\
& = \trace \left( \big( H(M) \big)^{-1} (B\transp B)^{-1} B\transp \Sigma_V B (B\transp B)^{-1} \right)\\
& = \trace \left( \big( H(M) \big)^{-1} (B\transp B)^{-1/2} \; \Sigma \; (B\transp B)^{-1/2} \right) \\
& = \trace \left( (B\transp B)^{-1/2} \big( H(M) \big)^{-1} (B\transp B)^{-1/2} \; \Sigma \right) \\
& = \frac{1}{K} \; \trace \left( \big( \Sigma(M) \big)^{1/2} \right) \; \trace \left( \big( \Sigma(M) \big)^{-1/2} \Sigma \right).
\end{aligned}
\end{equation}
Let \( S \Let \set[\big]{T \in \possemdef{n} \suchthat \image(T) = X_V } \). Since the sequence of maps
\begin{align*}
S \; \ni \; & T \longmapsto \Sigma(T) \; \in \posdef{m}, \\
\posdef{m} \; \ni \; & T \longmapsto T^{1/2} \; \in \posdef{m}, \\
\posdef{m} \; \ni \; & T \longmapsto T^{-1} \Sigma \; \in \posdef{m}, \\
\possemdef{m} \; \ni \; & T \longmapsto \trace(T) \; \in \R,
\end{align*}
are, evidently, continuous, it follows at once that the map \(S \ni M \longmapsto J(M) \in \R \) is also continuous. If \( \widehat{\Sigma}_V \) denotes the estimated second moment of \( V \), and the estimation is carried out with a large enough number of samples of \( V \), with probability one we have \( \image(\widehat{\Sigma}_V) = X_V \). Therefore, by continuity of \( M\mapsto J(M) \), we see at once that 
\[
	J(\widehat{\Sigma}_V) \xrightarrow[\widehat{\Sigma}_V \longrightarrow \Sigma_V]{} J(\Sigma_V) = \frac{\big( \trace(\Sigma^{1/2}) \big)^2}{K}.
\]
\end{remark}

\begin{remark}
	The optimal average cost of representation of a random vector $V$ is inversely proportional to the size \(K\) of the optimal dictionary, as is evident from the optimal costs in Theorems \ref{DL theorem} and \ref{DL general theorem}. To wit, the optimal average cost of representation decreases monotonically with $K$, which is expected.
\end{remark}

\begin{remark}
\label{scalar multiple}
$\ell_2$-optimal dictionaries for representing a random vector $V$ are also optimal for representing any scalar multiple \(\alpha V\) of $V$ for any $0 \neq \alpha \in \mathbb{R}$. Indeed, it is clear that $H^*$ defined in \eqref{H} is invariant under nonzero scalar multiplications of $V$. Therefore, \(\ell_2\)-optimal dictionaries are also invariant under nonzero scalar multiplications of the random vector $V$. This fact also follows from the observation made in Remark \ref{r:linearity of optimal schemes}.
\end{remark}

\begin{remark}
	An $\ell_2$-optimal dictionary as characterized by Theorem \ref{DL general theorem} appears there in the form of what is known as a \emph{rank-\(1\) decomposition} of the positive definite matrix $H^*$. Elements of the theory of rank-\(1\) decompositions of positive definite matrices is discussed below in Section \ref{matrices}. This particular decomposition plays a crucial r\^ole in transforming the search space of the $\ell_2$-optimal dictionary problem \eqref{DL problem} from the set of dictionaries to the set of symmetric positive definite matrices with real entries, and translating the non-convex $\ell_2$-optimal dictionary problem into a tractable convex one.
\end{remark}

\begin{remark}
	All $\ell_2$-optimal dictionaries are unique upto rank-\(1\) decompositions of a unique positive definite matrix that is obtained from the second moment $\EE[VV\transp]$ of the random vector $V$. That is, for a given random vector whose samples are to be optimally represented, every $\ell_2$-optimal dictionary is obtained from a rank-\(1\) decomposition of a unique positive definite matrix.
\end{remark}

\begin{remark}
	Looking ahead at Algorithm \ref{algo:general_l2_opt_algo}, it becomes evident that non-uniqueness of optimal dictionaries can be attributed to the non-uniqueness in the selection of $C$ in Step 5 of Algorithm \ref{algo:general_l2_opt_algo}, and the element of choice associated to the selection of $p_j$ and $p_k$ in Step 2 of Algorithm \ref{Algo1}. The number of optimal solutions may be infinite depending on the distribution of the random vector $V$. For instance, if $V$ is uniformly distributed over the unit sphere of $\mathbb{R}^n$ and $K = n$, then the elements in an $\ell_2$-optimal dictionary form an orthonormal basis of $\mathbb{R}^n$. (The special case of uniform distribution of $V$ over spheres is discussed in Section \ref{uniform section}.) Of course, there are infinitely many orthonormal bases of $\mathbb{R}^n$ for \(n \ge 2\).
\end{remark}

\begin{remark}
	From Algorithm \ref{algo:general_l2_opt_algo} on p.\ \pageref{algo:general_l2_opt_algo} we can infer that by calculating the matrix $B$ there, consisting of the eigenvectors of $\Sigma_V$ corresponding to its non-zero eigenvalues, the computations of $( B^\top B )^{-1/2}$, $\Sigma^{1/2}$, and $C$ in the decomposition given in Step 5 become straightforward. Therefore, the chief computational load in Algorithm \ref{algo:general_l2_opt_algo} consists of eigen-decomposition of $\Sigma_V$ and that in Algorithm \ref{Algo1} (in Step 6), both of which can be performed in polynomial time.
\end{remark}

\begin{example}
	\label{ex:1}
Let $V = \pmatr{V_1\\ V_2}$ be a random vector taking values in $\mathbb{R}^2$, with $V_1$ and $V_2$ being independent random variables. Let the density functions of $V_1$ and $V_2$ be 
\[
	\rho_{V_1}(v) = 2(v - 1) \mathds{1}_{[1,2]}(v) \quad\text{and}\quad
	\rho_{V_2}(v) = 2(2 - v) \mathds{1}_{[1,2]}(v),
\]
respectively. The support of \(V\) is, therefore, the square \([1, 2]\times[1, 2]\). Elementary calculations lead to $\Sigma_V \Let \EE_\rho [VV^\top] = \pmatr{17/6 & 20/9 \\ 20/9 & 11/6}$. We employed the procedure described in Algorithm \ref{procedure 2} for the given matrix $\Sigma_V$ and $K = 3$ in \textsc{matlab}. An optimal dictionary \(\{y_1\opt, y_2\opt, y_3\opt\}\) was obtained, with
\[
	y_1^* = \pmatr{0.9789 \\ 0.2045},\quad y_2^* = \pmatr{0.6792 \\ 0.7339},\quad y_3^* = \pmatr{0.5870 \\ 0.8096};
\]
the optimum value of the objective function was reported to be \(1.8930\). This collection \(\{y_i\opt\}_{i=1}^3\) of optimal vectors are marked with crosses on the circumference of the unit circle shown in Figure \ref{figure1}. A second optimal dictionary $\{z_1^*,z_2^*,z_3^*\}$ was obtained, also using Algorithm \ref{procedure 2}, with dictionary vectors
\[
	z_1^* = \pmatr{0.4214 \\ 0.9069},\quad z_2^* = \pmatr{0.9284 \\ 0.3717},\quad z_3^* = \pmatr{0.8513 \\ 0.5247},
\]
with an identical optimal value as in the former case. The vectors \(\{z_i\opt\}_{i=1}^3\) are marked with dark circles on the circumference of the unit circle in Figure \ref{figure1}.
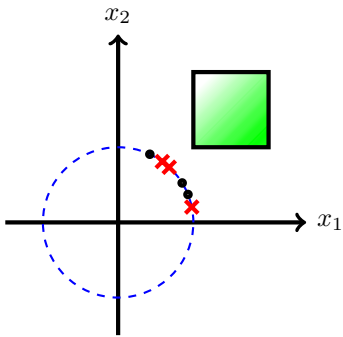
\begin{figure}[H]
\begin{center}
\begin{tikzpicture}[scale=1]
\draw[->,ultra thick] (-1.5,0)--(2.5,0) node[right]{$x_1$};
\draw[->,ultra thick] (0,-1.5)--(0,2.5) node[above]{$x_2$};
\draw[blue,thick,dashed] (0,0) circle (1cm);
\shade[top color=white, bottom color=green,shading angle=45,line width=2mm] (1,1) rectangle (2,2);
\draw[ultra thick](1,1) rectangle (2,2);
\draw[ultra thick](0.4214,0.9069) circle(1pt);
\draw[ultra thick](0.9284,0.3717) circle(1pt);
\draw[ultra thick](0.8513,0.5247) circle(1pt); 
\draw[ultra thick] (0.9789,0.2045) node[cross=4pt,red] {};
\draw[ultra thick] (0.6792 ,0.7339) node[cross=4pt,red] {};
\draw[ultra thick] (0.5870,0.8096) node[cross=4pt,red] {};
\end{tikzpicture}
\caption{The two optimal dictionaries in Example \ref{ex:1}.}
\label{figure1}
\end{center}
\end{figure}
It is expected that the optimal dictionary vectors are concentrated towards the bottom right corner of the support \([1, 2]\times[1, 2]\) (the region with strong shading in figure \ref{figure1}). In the optimal solution $\{z^*_i\}_{i = 1}^3$, two vectors \(z_2^*\) and \(z^*_3\) point to the region where density of $V$ is concentrated the most. Also, for the solution $\{y^*_i\}_{i = 1}^3$, two vectors \(y^*_2\) and \(y^*_3\) are oriented towards the center of the square \([1, 2]\times[1, 2]\), with the remaining vector pointing towards the region of higher density. These results correlate positively with what may be expected out of \(\ell_2\)-optimal dictionaries.
\end{example}

\subsection{Uniform distribution over the unit sphere}
\label{uniform section}
We shall test our results on the important case of \(\mu\) being the uniform distribution on the unit sphere. Note that due to (rigid) rotational symmetry of the distribution, it follows that rigid rotations of optimal dictionaries in this case are also optimal.

Let us consider a dictionary consisting of (unit) vectors that are `close' to each other, i.e., the inner product between any two elements of the dictionary is close to \(1\). It is quite evident that such a dictionary is not optimal for representing uniformly distributed samples due to the fact that samples of \(V\) that are almost orthogonal to the dictionary vectors carry equal priority as any other vector but require large coefficients for their representation. It is, therefore, more natural to search for dictionaries in which the constituent vectors are `maximally spaced out'.

Several examples of collections of vectors that are `maximally spaced out' may be found in \cite[Section 4]{benedetto2003finite}. Collections of vectors that are maximally far apart from each other are known to attain `equilibria' under the actions of different kinds of forces defined and explained in \cite[Section 4]{benedetto2003finite} and \cite[p.\ 6]{saff1997distributing}. Such collections of vectors are generalized by the ideas of \emph{tight frames} as explained in \cite{benedetto2003finite}; see also \cite{ref:Chr-16,daubechies1986painless, benedetto2003finite, zimmermann2001normalized} for related information.

We recall here some standard definitions for completeness and to provide the necessary substratum for our next result. Let $n,K$ be positive integers such that $K \geq n$. We say that a collection of vectors $\{ x_i \}_{i = 1}^K$ is a \emph{frame} for $\R^n$ if there exist some constants $c, C > 0$ such that
\[
	c \norm{x}^2 \leq \sum_{i = 1}^K \inprod{x_i}{x}^2 \leq C \norm{x}^2 \quad \text{for all $x \in \R^n$.}
\]
We say that a frame $\{x_i\}_{i = 1}^K\subset\R^n$ is \emph{tight} if $c = C$. In addition, if $\{x_i\}_{i = 1}^K\subset\R^n$ is a tight frame and $\norm{x_i} = 1$ for all $i = 1,2,\ldots,K$, we say that the collection $\{x_i\}_{i = 1}^K$ is a $c$-\emph{unit norm tight frame} (a \(c\)-UNTF).

We have the following connection between \(\ell_2\)-optimal dictionaries and UNTFs:
\begin{proposition}
	A dictionary $D_K = \{\dict{i}\}_{i = 1}^K$ is optimal for representing samples of a random vector $V$ that is uniformly distributed over the surface of the unit sphere of \(\R^n\) if and only if the collection \(\{\dict{i}\}_{i=1}^K\) of vectors constitute a $\frac{K}{n}$-UNTF.
\end{proposition}
\begin{proof}
	If $V$ is uniformly distributed over the unit sphere, we have $\Sigma_V = \EE[V V\transp] = \frac{1}{n} I_n$. According to Theorem \ref{DL theorem} the collection $\{\dict{i}\}_{i = 1}^K$ is an optimal dictionary if and only if 
\begin{equation}
\label{uniform 1}
\sum_{i = 1}^K \dict{i}\dict{i}^\top = \frac{K}{\trace\bigl(\frac{1}{\sqrt{n}}I_n\bigr)}\Bigl(\frac{1}{\sqrt{n}}I_n\Bigr) = \frac{K}{n}I_n.
\end{equation}
Since the family \(\{\dict{i}\}_{i=1}^K\) must span \(\R^n\) by definition, it is a frame. The \emph{frame operator} for the frame $\{\dict{i}\}_{i = 1}^K$ is given by \cite[Section 2]{benedetto2003finite}
\[
	\R^n\ni y\mapsto S(y) \Let \sum_{i = 1}^K \inprod{\dict{i}}{y} \dict{i} = \biggl(\sum_{i = 1}^K \dict{i}\dict{i}^\top \biggr) y \in\R^n,
\]
where \(\inprod{v}{w} = v\transp w\) is the standard inner product in \(\R^n\). \cite[Theorem 3.1]{benedetto2003finite} asserts that a collection of unit norm vectors $\{\dict{i}\}_{i = 1}^K$ forms a tight frame in $\mathbb{R}^n$ if and only if the collection is a $\frac{K}{n}$-UNTF. From \cite[Theorem 2.1]{benedetto2003finite} it follows that a collection of vectors $\{\dict{i}\}_{i = 1}^K$ is a $\frac{K}{n}$-UNTF if and only if 
\begin{equation}
\label{uniform 2}
	S = \sum_{i = 1}^K  \dict{i}\dict{i}^\top  = \frac{K}{n}I_n.
\end{equation}
The assertion follows from \eqref{uniform 1} and \eqref{uniform 2}.
\end{proof}

\section{A particular class of rank-\(1\) decompositions of matrices}
\label{matrices}
We collect and establish here some results on the theory of rank-\(1\) decompositions of matrices. While these facts will be needed for our main results, they are also of independent interest.

A standard result in matrix theory \cite[p.\ 2]{bhatiapositive} states that a symmetric positive semidefinite matrix with real entries $M \in \possemdef{n}$, can be decomposed as $ Y Y^\top $ for some $ Y \in \mathbb{R}^{n \times r}$, where $r \Let \rank(M)$. Let $y_i$ indicate the $i$ th column of the matrix $Y$. Then the equality \(M = Y Y\transp\) is equivalent to
\[
	M = \sum_{i = 1}^r y_iy_i^\top.
\]
More generally for $K \geq r$, let 
\[
\overline{M} \Let \begin{pmatrix}
	M & O_{n\times (K-r)} \\
	O_{(K-r)\times n} & I_{K-r}
\end{pmatrix},
\]
where $ O $ is a zero matrix of order $n \times (K - r)$. If we consider the decomposition of $\overline{M}$ as $ \overline{M} = \overline{Y}~ \overline{Y}^\top $ with $\overline{Y} \in \mathbb{R}^{(n + K - r) \times K}$, and indicate by $Y$ the upper $ n \times K $ matrix block of $ \overline{Y} $, we get $ M = Y Y^\top $. In other words 
\begin{equation}
\label{decomposition 1}
M = \sum_{i = 1}^K y_iy_i^\top.
\end{equation}
There are numerous ways of decomposing positive semidefinite matrices; some of them are discussed in \cite[Theorem 7.3]{zhangmatrix}. The speciality of a particular decomposition lies in the characteristics exhibited by the vectors $ y_i $'s. A particular rank-\(1\) decomposition which we will use to solve the \( \ell_2 \)-optimal dictionary problem is the one where for every \( M \in \possemdef{n} \) and \( K \geq r \Let \rank(M) \) there exists a collection of vectors \( \{ y_i \}_{i = 1}^K  \subset \R^n \) that satisfy
\begin{equation}
\label{eq:rod-decomp}
M = \sum_{i = 1}^K y_i y_i\transp \quad \text{and} \quad y^\top_i y_i = \frac{\trace(M)}{K} \quad \text{for all $i = 1,\ldots,K$}.
\end{equation}
We are now in a position to present Algorithm \ref{Algo1} and its associated Theorem \ref{theorem 1}, whose corollaries will give us the needed rank-\(1\) decomposition of \eqref{eq:rod-decomp}. We mention that Algorithm \ref{Algo1} is, in principle, similar to Procedure 1 of \cite{sturm2003cones}, and in particular, the assertions of Theorem \ref{theorem 1} and its corollaries can be obtained by applying \cite[Proposition 3 and Corollary 4]{sturm2003cones} via some straightforward modifications. However, we provide the complete proofs here for the sake of completeness.

\begin{algorithm}[h]
	\KwIn{A matrix $\Lambda\in\R^{n\times n}$.}
\KwOut{An orthonormal collection of vectors $ (x_i)_{i = 1}^n \subset \mathbb{R}^n$ such that $x_i^\top \Lambda x_i  = \frac{\trace(\Lambda)}{n}$ for all $i = 1, \ldots,n$.}
\nl Initialize quantities by $S_0 = \emptyset $, $i = 1$. 

\nl \textbf{for} $i$ from 1 to $(n - 1)$ 

\nl \textbf{do}

\quad $ S'_i = S_{i-1} \uplus (e_1,e_2,\ldots ,e_n) $.

\quad $P_i = \Ortho(S'_i) \setminus S_{i-1} $.

\quad Find $p_j,p_k \in P_i$ such that $p_j^\top \Lambda p_j \leq \frac{\trace(\Lambda)}{n} \leq p_k^\top \Lambda p_k$. 

\quad Let $\Theta \in [0,1]$ be a solution of the equation (in \(\theta\))
\[
	\bigl((1 - \theta)p_j + \theta p_k\bigr)^\top \Lambda \bigl((1 - \theta)p_j + \theta p_k\bigr) = \frac{\trace(\Lambda)}{n} \bigl((1 - \theta)^2 + \theta^2\bigr)
\]


\quad Define $x_i \Let \frac{(1 - \Theta)p_j + \Theta p_k }{\left( (1 - \Theta)^2 + \Theta^2\right)^{1/2} }$. 

\quad Define $S_{i} \Let S_{i-1} \uplus (x_i)$. 

\nl \textbf{end for loop}

\nl $ S'_n = S_{n-1} \uplus (e_1,e_2,\ldots ,e_n ) $.

\nl Output $S_{n} \Let \Ortho(S'_n) $.
\caption{Calculation of orthonormal bases \`a la Theorem \ref{theorem 1}}
\label{Algo1}
\end{algorithm}

\begin{theorem}\label{theorem 1}
	For any matrix $\Lambda \in \mathbb{R}^{n \times n}$ there exists an orthonormal collection \( (x_i )_{i=1}^n\subset\R^n\) of vectors satisfying
\[
	x_i^\top \Lambda x_i = \frac{\trace(\Lambda)}{n} \\ \quad \;\; \text{for all}\;\; i = 1, \ldots,n,
\]
Moreover, such a collection can be obtained from Algorithm \ref{Algo1}.
\end{theorem}
\begin{proof}
First we establish that the collection of vectors $ ( x_i )_{i = 1}^{n - 1}$ contained in $S_{n - 1}$ (recall that \(S_{n-1}\) is generated in the \textbf{for} loop in the Algorithm \ref{Algo1},) are orthonormal, and satisfy $x_i^\top \Lambda x_i = \frac{\trace(\Lambda)}{n}$ for $i = 1,\ldots, n - 1$. We shall prove this by induction on $i$.

\textsf{The induction base:} For $i = 1$, we have $P_1 = (e_1,e_2,\ldots ,e_n) $. Since $\sum_{m = 1}^n e_m^\top \Lambda e_m = \trace(\Lambda)$, vectors $p_j,p_k \in P_1$ exist such that $p_j^\top \Lambda p_j \leq  \frac{\trace(\Lambda)}{n} \leq p_k^\top \Lambda p_k$. We solve for $\theta$ in the equation
\begin{equation}
\label{theorem 1 eq1}
g_{p_j ; p_k}(\theta) \Let \frac{\left((1 - \theta)p_j + \theta p_k)^\top \Lambda ((1 - \theta)p_j + \theta p_k \right)}{\left( (1 - \theta)^2 + \theta^2 \right)}
 = \frac{\trace(\Lambda)}{n}.
\end{equation}
We know that a solution exists in $[0,1]$ because for $\theta = 0$ we have
\[
g_{p_j ; p_k}(0) = \left[ \frac{((1 - \theta)p_j + \theta p_k)^\top \Lambda ((1 - \theta)p_j + \theta p_k)}{\left((1 - \theta)^2 + \theta^2\right)} \right]_{\theta = 0}
 = p_j^\top \Lambda p_j \leq \frac{\trace(\Lambda)}{n},
\]
for $\theta = 1$ we have
\[
g_{p_j ; p_k}(1) = \left[ \frac{((1 - \theta)p_j + \theta p_k)^\top \Lambda ((1 - \theta)p_j + \theta p_k)}{\left((1 - \theta)^2 + \theta^2\right)} \right]_{\theta = 1}
 = p_k^\top \Lambda p_k \geq \frac{\trace(\Lambda)}{n},
\]
and $g_{p_j ; p_k}(\cdot)$ is a continuous function of $\theta$. Let $\Theta$ be such a solution. Then, following the notation in Algorithm \ref{Algo1}, we have
\[
x_1 \Let \frac{(1 - \Theta)p_j + \Theta p_k }{\sqrt{(1 - \Theta)^2 + \Theta^2}}.
\]
Since $p_j,p_k$ are elements of $P_1$, they are orthonormal; therefore,
\[
\norm{x_1} = \frac{\sqrt{(1 - \Theta)^2 \norm{p_j}^2 + \Theta^2 \norm{p_k}^2}}{\sqrt{(1 - \Theta)^2 + \Theta^2}} = 1,
\]
and since $\Theta$ is a solution of equation \eqref{theorem 1 eq1} we have
\[
x_1^\top \Lambda x_1 = \frac{\trace(\Lambda)}{n}.
\]

\textsf{Induction hypothesis:} Assume that for some $i$ between \(1\) and \(n-1\) the collection $S_i = ( x_\ell )_{\ell = 1}^i$ is orthonormal, and satisfies 
\[
x_\ell^\top \Lambda x_\ell = \frac{\trace(\Lambda)}{n} \;\;  \text{ for all $\ell = 1,\ldots,i$.}
\]

\textsf{Induction step:} In view of the induction hypothesis, we define
\begin{align*}
	S'_{i + 1}	& \Let S_{i} \uplus ( e_1,e_2,\ldots ,e_n ) = ( x_1,x_2,\ldots x_i,e_1,e_2,\ldots,e_n ),\\
	\intertext{and compute}
	\Ortho(S'_{i + 1}) & = ( x_1,x_2,\ldots,x_i,p_1,p_2,\ldots,p_{n-i} ),\\
	P_{i+1} & = ( p_1,p_2,\ldots,p_{n-i} )
\end{align*}
as in Algorithm \ref{Algo1}. Since the collection $ ( x_\ell )_{\ell = 1}^i \uplus ( p_\ell )_{\ell = 1}^{n -i}$ is an orthonormal basis for $\mathbb{R}^n$, we have
\[
\sum_{\ell = 1}^i x_\ell^\top \Lambda x_\ell + \sum_{\ell = 1}^{n -i} p_\ell^\top \Lambda p_\ell = \trace(\Lambda),
\]
leading to
\[
\sum_{\ell = 1}^{n - i} p_\ell^\top \Lambda p_\ell = \frac{(n - i)}{n}\trace(\Lambda).
\]
Thus, there exist vectors $p_j,p_k \in P_{i + 1}$ such that $p_j^\top \Lambda p_j \leq  \frac{\trace(\Lambda)}{n} \leq p_k^\top \Lambda p_k$.  Let us consider the equation
\begin{equation}
\label{theorem 1 eq2}
g_{p_j,p_k}(\theta) \Let \frac{((1 - \theta)p_j + \theta p_k)^\top \Lambda ((1 - \theta)p_j + \theta p_k)}{\left((1 - \theta)^2 + \theta^2\right)}
 = \frac{\trace(\Lambda)}{n}
\end{equation}
in \(\theta\). From arguments given in the case of $i = 1$, we know that a solution $\Theta$ of \eqref{theorem 1 eq2} exists on $[0,1]$. We define
\[
x_{i + 1} \Let \frac{(1 - \Theta)p_j + \Theta p_k }{\sqrt{(1 - \Theta)^2 + \Theta^2}}.
\]
Since $p_j,p_k$ are orthogonal to the vectors $ ( x_\ell )_{\ell = 1}^{i}$, so is any linear combination of $p_j, p_k$. Therefore, $x_{i + 1}$ is orthogonal to the vectors $( x_\ell )_{\ell = 1}^{i}$, which, along with the fact that
\[
\norm{x_{i + 1}} = \frac{\sqrt{(1 - \Theta)^2 \norm{p_j}^2 + \Theta^2 \norm{p_k}^2}}{\sqrt{(1 - \Theta)^2 + \Theta^2}} = 1,
\]
makes the collection $( x_\ell )_{\ell = 1}^{i + 1}$ orthonormal. Also, since $\Theta$ is a solution of \eqref{theorem 1 eq2}, we get 
\[
	x_{i + 1}^\top \Lambda x_{i + 1} = \frac{\trace(\Lambda)}{n}.
\]
Therefore, by mathematical induction, we conclude that the collection $ ( x_i )_{i = 1}^{n - 1}$ contained in $S_{n - 1}$ has the required properties.

Finally, in the 4th and 5th steps of Algorithm \ref{Algo1}, we get
\[
	S'_n = ( x_1,x_2,\ldots, x_{n - 1},e_1,e_2,\ldots,e_n ),
\]
and
\[
	\Ortho(S'_n) = ( x_1,x_2,\ldots, x_{n - 1},x_n ).
\]
By construction, $( x_\ell )_{\ell = 1}^{n}$ is an orthonormal collection, implying that $\sum_{i = 1}^n x_i^\top \Lambda x_i = \trace(\Lambda)$. In turn, this leads to
\[
\begin{aligned}
x_n^\top \Lambda x_n &= \sum_{i = 1}^n x_i^\top \Lambda x_i - \sum_{i = 1}^{n - 1} x_i^\top \Lambda x_i \\
&= \trace(\Lambda) - \Bigl(\frac{n -1}{n}\Bigr)\trace(\Lambda) \\
&= \frac{\trace(\Lambda)}{n}.
\end{aligned}
\]
Thus, Algorithm \ref{Algo1} yields a collection of orthonormal vectors $ ( x_i )_{i = 1}^n $ such that 
\[
	x_i^\top \Lambda x_i = \frac{\trace(\Lambda)}{n} \quad \text{for all $i = 1,2,\ldots,n$},
\]
thereby completing the proof.
\end{proof}

\begin{corollary}[Rank-\(1\) decomposition]\label{ROD theorem}
	Let $X \in \possemdef{n}$, define \(r\Let \rank(X)\), and let $ T \in \mathbb{S}^{n \times n}$. There exists a collection of vectors $\{x_i\}_{i = 1}^r \subset \mathbb{R}^n$ such that
	\[
		X = \sum_{j = 1}^r x_jx_j^\top, \quad\text{and}\quad x_i^\top T x_i = \frac{1}{r}\trace(XT) \,\text{ for all } \,i = 1,\ldots,r.
	\]
\end{corollary} 
\begin{proof}
We know \cite[p.\ 2]{bhatiapositive} that any symmetric positive semidefinite matrix $X$ with real entries and of rank $r$ can be decomposed as $CC^\top $ where $C \in \mathbb{R}^{n \times r}$. Let us define $\Lambda \in \mathbb{R}^{r \times r}$ as $\Lambda \Let C^\top T C$. According to Theorem \ref{theorem 1} a collection of orthonormal vectors $\{y_i\}_{i = 1}^r \subset \mathbb{R}^r$ can be obtained such that
\[
y_i^\top C^\top T \;C y_i = y_i^\top \Lambda y_i = \frac{\trace(\Lambda)}{r}.
\]
We define a collection $\{x_i\}_{i = 1}^r\subset\R^r$ by $x_i \Let C y_i$ for $i = 1,\ldots,r$. Then
\[
	\sum_{i = 1}^r x_ix_i^\top = C \biggl(\sum_{i = 1}^r y_i y_i^\top\biggr) C^\top = C I_r C^\top = X.
\]
Moreover, for every $i = 1,\ldots,r$,
\[
	x_i^\top T x_i = y_i^\top C^\top T C y_i = \frac{\trace(\Lambda)}{r} = \frac{1}{r}\trace(C^\top T C) = \frac{1}{r}\trace(X T).
\]
The assertion follows.
\end{proof}

Corollary \ref{ROD theorem} is generalized slightly by the following one; we shall employ this particular form to solve the $\ell_2$-optimal dictionary problem in Theorem \ref{DL theorem}.

\begin{corollary}
\label{theorem 2}
Let $ M \in \possemdef{n}$ and define \(r \Let \rank(M)\). Let $A \in \mathbb{S}^{n \times n}$ and $K \geq r$ be given. There exists a collection of vectors $\{y_i\}_{i = 1}^K \subset \mathbb{R}^n$ such that
\begin{equation}
	\label{decomposition}
	M = \sum_{j = 1}^K y_jy_j^\top, \quad\text{and}\quad y_i^\top A y_i = \frac{1}{K}\trace(MA) \,\text{ for all }\, i = 1,\ldots,K.
\end{equation}
\end{corollary}
\begin{proof}
Let us consider the square matrices $X,T$ of order $K+n-r$ in Corollary \ref{ROD theorem} to be 
\[
	X \Let \begin{pmatrix}
		M & O_{n\times(K-r)} \\
		O_{(K-r)\times n} & I_{K-r}
	\end{pmatrix}
	\quad\text{and}\quad
	T \Let \begin{pmatrix}
		A & O_{n\times(K-r)} \\
		O_{(K-r)\times n} & O_{(K-r)\times (K-r)}
	\end{pmatrix}.
\]
Then $\rank(X) = K$ by construction. Therefore, vectors $\{x_i\}_{i = 1}^K \subset \mathbb{R}^{n + K - r}$ exist satisfying the properties in Corollary \ref{ROD theorem}. Let us denote $\mathbb{R}^n \ni y_i \Let \pmatr{x_{i1}& \ldots & x_{in}}^\top$ for $i = 1,\ldots, K$; in other words, $y_i$ is the vector formed by the first $n$ components of $x_i$. Then
\[
	\sum_{i = 1}^K y_iy_i^\top = M,
\]
and for any $i = 1, \ldots, K$,
\[
	y_i^\top A y_i = x_i^\top T x_i = \frac{1}{K} \trace(XT) = \frac{1}{K} \trace(MA).
\]
The assertion follows at once.
\end{proof}

\section{Proofs of Theorem \ref{DL theorem}, Lemma \ref{lemma 2}, and Theorem \ref{DL general theorem}}
\label{Proofs}

\subsection{Proof of Theorem \ref{DL theorem}}
\label{sec:proof-of-DL-theorem}
\begin{proof}
For a given dictionary $D_K \in \mathcal{D}_K$ of vectors $\{ \dict{i} \}_{i = 1}^K$ that is feasible for \eqref{DL problem}, let us define a scheme of representation
\[
	\R^n\ni v\mapsto f^*_{D_K}(v) \Let \pmat{ \dict{1} }{ \dict{2} }{ \dict{K} }^+ v\in\R^K.
\]
Quite clearly, $\pmat{ \dict{1} }{ \dict{2} }{ \dict{K} } f^*_{D_K}(v) = v $ for any $v \in \mathbb{R}^n$ by the definition of the pseudo-inverse because if $\Span\{\dict{i}\}_{i = 1}^K = \mathbb{R}^n$, then 
 $\pmat{ \dict{1} }{ \dict{2} }{ \dict{K} }^+ v$ solves the equation $\pmat{ \dict{1} }{ \dict{2} }{ \dict{K} }x = v$. Therefore,
\[
	\pmat{ \dict{1} }{ \dict{2} }{ \dict{K} } f^*_{D_K}(V) = V \text{\ \ \(\mu\)-almost surely}.
\]
We know that $f^*_{D_K}(v) = \pmat{ \dict{1} }{ \dict{2} }{ \dict{K} }^+ v$ is the solution of the least squares problem
\[
\begin{aligned}
	& \minimize_{x\in\R^K}	&& \norm{x}^2\\
	& \sbjto				&& \pmat{ \dict{1} }{ \dict{2} }{ \dict{K} } x = v.
\end{aligned}
\]                
Therefore, for an arbitrary $f \in \mathcal{F}$ such that $\pmat{ \dict{1} }{ \dict{2} }{ \dict{K} } f(v) = v $ for all $v \in \mathbb{R}^n $, we must have
\begin{align*}
\norm{f^*_{D_K}(v)}^2 & \leq \norm{f(v)}^2 \quad\text{for all } v \in \mathbb{R}^n.
\end{align*}
Therefore,
\[
	\norm{f^*_{D_K}(V)}^2 \leq \norm{f(V)}^2 \text{\ \ \(\mu\)-almost surely},
\]
and hence,
\[
	\EE_\mu \bigl[ \norm{f^*_{D_K}(V)}^2 \bigr] \leq \EE_\mu \bigl[ \norm{f(V)}^2 \bigr].
\]
Minimizing over all feasible dictionaries and schemes, we get
\begin{equation}
	\label{first comparison}
	\inf_{D_K \in \mathcal{D_K}} \EE_\mu \bigl[ \norm{f^*_{D_K}(V)}^2 \bigr] \leq \inf_{\substack{D_K\in\mathcal{D}_K,\\ f\in\mathcal F}} \EE_\mu \bigl[ \norm{f(V)}^2 \bigr] \\
\end{equation}
The problem on the left-hand side of the inequality \eqref{first comparison} is
\begin{equation}
	\label{first DL}
	\begin{aligned}
		& \minimize_{\{\dict{i}\}_{i=1}^K} && \EE_\mu \bigl[ \norm{f_{D_K}^*(V)}^2 \bigr] \\
		& \sbjto		&&
		\begin{cases}
			\norm{ \dict{i} } = 1 \text{ for all $i = 1,\ldots,K$,}\\
			\Span\{ \dict{i} \}_{i = 1}^K = \mathbb{R}^n. 
		\end{cases}
	\end{aligned}
\end{equation}
From \eqref{first comparison} we can conclude that the optimal value, if it exists, of problem \eqref{DL problem} is bounded below by the optimal value, if it exists, of the one given in \eqref{first DL}. Our strategy is to demonstrate that optimization problem \eqref{first DL} admits a solution, and we shall furnish a feasible solution of \eqref{DL problem} that achieves a value of the objective function that is equal to the optimal value of the problem \eqref{first DL}. This will solve \eqref{DL problem}.

Let $D \Let \pmat{\dict{1}}{\dict{2}}{\dict{K}}$. The objective function in \eqref{first DL} can be computed as
\begin{align*}
\EE_\mu \bigl[ \norm{f_{D_K}^*(V)}^2 \bigr] &= \EE_\mu \bigl[ \norm{D^+V}^2 \bigr] \\
&= \EE_\mu \bigl[ V^\top (D^+)^\top D^+V \bigr] \\
&= \EE_\mu \bigl[  V^\top \bigl( D^\top (DD^\top )^{-1}\bigr)^\top \bigl( D^\top (DD^\top )^{-1}\bigr)V\bigr] \\
&= \EE_\mu \bigl[ V^\top (DD^\top )^{-1}DD^\top (DD^\top )^{-1}V \bigr] \\
&= \EE_\mu \bigl[ V^\top (DD^\top )^{-1}V \bigr] \\
&= \EE_\mu \bigl[ \trace(V^\top (DD^\top )^{-1}V) \bigr] \\
&= \EE_\mu \bigl[ \trace(VV^\top (DD^\top )^{-1}) \bigr] \\
&= \trace \left(\EE_\mu \bigl[ VV^\top \bigr](DD^\top )^{-1}\right).
\end{align*} 
Letting $ \Sigma_V \Let \EE_\mu \bigl[VV^\top \bigr] $ and writing $ DD^\top = \sum_{i = 1}^K \dict{i}  \dict{i}^\top $ the optimization problem \eqref{first DL} is rephrased as
\begin{equation}
	\label{DL 1}
	\begin{aligned}
		& \minimize_{\{\dict{i}\}_{i=1}^K}		&& \trace\biggl( \Sigma_V \biggl(\sum_{i = 1}^K  \dict{i}  \dict{i}^\top \biggr)^{-1} \biggr)\\
		& \sbjto		&&
		\begin{cases}
			\norm{ \dict{i} } = 1 \text{ for all $i = 1, \ldots, K$,} \\
			\Span\{ \dict{i} \}_{i = 1}^K = \mathbb{R}^n.
		\end{cases}
	\end{aligned}
\end{equation}
Let $S$ be the feasible set for the problem in \eqref{DL 1}. At first \eqref{DL 1} appears to be non-convex. Let us demonstrate that the objective function of \eqref{DL 1} is convex in \(DD\transp\). We know that whenever $\Sigma_V$ is a positive definite matrix, 
\[
	\trace(\Sigma_V M^{-1}) = \trace\Bigl(\Sigma_V^{1/2}M^{-1}\Sigma_V^{1/2}\Bigr) = \trace\Bigl(\bigl(\Sigma_V^{-1/2}M\Sigma_V^{-1/2}\bigr)^{-1}\Bigr).
\]
From \cite[p.\ 113 and Exercise V.1.15, p.\ 117]{bhatiamatrix} we know that inversion of a matrix is a \emph{matrix convex} map on the set of positive definite matrices. Therefore, for any $\theta \in [0,1]$ and $M_1,M_2 \in \posdef{n}$ we have
\begin{multline}
\left( \Sigma_V^{-1/2} \bigl( (1-\theta) M_1 + \theta M_2 \bigr) \Sigma_V^{-1/2} \right)^{-1}\\
= \left( (1-\theta) \left( \Sigma_V^{-1/2} M_1 \Sigma_V^{-1/2} \right) + \theta \left( \Sigma_V^{-1/2} M_2 \Sigma_V^{-1/2} \right) \right)^{-1}\\
\preceq (1-\theta) \left( \Sigma_V^{-1/2} M_1 \Sigma_V^{-1/2} \right)^{-1} + \theta \left( \Sigma_V^{-1/2} M_2 \Sigma_V^{-1/2} \right)^{-1},
\end{multline}
where $A \preceq B$ implies that $B - A$ is positive semidefinite. Since $\trace(\cdot)$ is a \emph{linear functional} over the set of $n \times n$ matrices we have
\begin{multline*}
	\trace \Bigl( \Sigma_V \bigl( (1-\theta) M_1 + \theta M_2 \bigr)^{-1} \Bigr) = \trace\Bigl(\Bigl( \Sigma_V^{-1/2} \bigl( (1-\theta) M_1 + \theta M_2 \bigr) \Sigma_V^{-1/2} \Bigr)^{-1}\Bigr) \\
\leq (1-\theta) \trace \Bigl(\Bigl( \Sigma_V^{-1/2} M_1 \Sigma_V^{-1/2} \Bigr)^{-1}\Bigr) + \theta \trace\Bigl(\Bigl( \Sigma_V^{-1/2} M_2 \Sigma_V^{-1/2} \Bigr)^{-1}\Bigr) \\
\leq (1-\theta) \trace (\Sigma_V M_1^{-1}) + \theta\trace (\Sigma_V M_2^{-1}).
\end{multline*}
In other words, the function $ M\mapsto \trace(\Sigma_V M^{-1}) $ is a convex function on the set of symmetric and positive definite matrices. Moreover, we know that for a collection $\{ \dict{i} \}_{i = 1}^K$ that is feasible for \eqref{DL 1},
\[
	\mathcal D_K\ni \{\dict{i}\}_{i=1}^K\mapsto h( \dict{1} , \ldots,  \dict{K} ) \Let \sum_{i = 1}^K  \dict{i}  \dict{i}^\top 
\]
maps into the set of positive definite matrices. Therefore, the objective function in \eqref{DL 1} is a convex function on $ \image(h) $. This allows us to translate the feasible set of the optimization problem \eqref{DL 1} to the set of matrices \(M\) formed by all feasible collections $\{ \dict{i} \}_{i = 1}^K$, i.e., on \(h(\mathcal D_K)\).

Let $ R \Let \set[\big]{M \in \posdef{n} \suchthat \trace(M) = K }$. On the one hand, from Corollary \ref{theorem 2} with $A = I_n$, we know that any symmetric and positive definite matrix $ M \in R $ can be decomposed as 
\[
	M = \sum_{i = 1}^K \dict{i}\dict{i}^\top \quad \text{with } \norm{ \dict{i} } = \sqrt{\frac{\trace(M)}{K}} = 1\text{ for all }i =  1, \ldots, K.
\]
The fact that $M$ is positive definite implies that $\Span\{\dict{i}\}_{i = 1}^K = \mathbb{R}^n$. Therefore, $\{\dict{i}\}_{i = 1}^K \in \mathcal D_K $ and $ M = h(\dict{1}, \ldots, \dict{K})$, which implies that
\begin{equation}
	\label{4}
	R \subset h(S).
\end{equation}
On the other hand, for any collection of vectors $\{\dict{i}\}_{i = 1}^K \in \mathcal D_K$, we have $h(\dict{1}, \ldots, \dict{K}) = \sum_{i = 1}^K \dict{i}\dict{i}^\top \in \posdef{n}$ and $ \trace\bigl( h(\dict{1}, \ldots, \dict{K}) \bigr) = \sum_{i = 1}^K \dict{i}^\top \dict{i} = K $. Therefore, by definition of $ R $,
\begin{equation}
	\label{5}
	h(S) \subset R.
\end{equation}
From \eqref{4} and \eqref{5} we conclude that $ h(\mathcal D_K) = R $. The optimization problem \eqref{DL 1} is, therefore, equivalent to the one where the feasible set is the set of positive definite matrices with trace \(K\), i.e., from \eqref{DL 1},
\begin{equation}
	\label{DL 2}
	\begin{aligned}
		& \minimize_{M \in \posdef{n}}	&& \trace \bigl(\Sigma_VM^{-1}\bigr)\\
		& \sbjto	&& \trace(M) - K = 0.
	\end{aligned}
\end{equation}

The optimization problem in \eqref{DL 2} is convex since its objective function is convex (as a function of $M$) and the feasible region is the intersection of a convex cone $\posdef{n}$ and the affine space $\set[\big]{M \in\R^{n\times n}\suchthat\trace(M) - K = 0}$. In the light of \cite[p.\ 244]{boydconvex} it follows that \eqref{DL 2} can be solved by considering just the first order optimality conditions. These first order optimality conditions are expressed in terms of a Lagrangian
\[
	L(M,\gamma) \Let \trace(M^{-1} \Sigma_V) + \gamma \bigl(\trace(M) - K \bigr),
\]
containing a KKT multiplier $\gamma$ at an optimal point \(M^*\) as
\begin{equation}
\label{opt1 DL 2}
\begin{aligned}
0 = \nabla_M L(M^*,\gamma) &= \nabla_M \Bigl(\trace(M^{-1} \Sigma_V) + \gamma\bigl(\trace(M) - K \bigr)\Bigr) \bigg|_{M = M^*} \\
& = - \bigl( (M^*)^{-1} \Sigma_V (M^*)^{-1} \bigr)^\top + \gamma I_n.
\end{aligned}
\end{equation}
But since $M^*, \Sigma_V \in \posdef{n}$, by symmetry it follows that \( (M^*)^{-1} \Sigma_V (M^*)^{-1} = \gamma I_n\), leading to
\begin{equation}
	\label{converge point}
	\Sigma_V = \gamma (M^*)^2.
\end{equation}
Since $\Sigma_V \neq O_{n \times n}$, we get $\gamma \neq 0$, and write $M^*$ as
\[
	M^* = \frac{1}{\sqrt{\gamma}} \Sigma^{1/2}_V.
\]
To evaluate $\gamma$ we use the fact that by construction \(K =\trace(M^*) = \frac{1}{\sqrt{\gamma}} \trace \bigl( \Sigma^{1/2}_V \bigr)\), which gives 
\[
\gamma = \bigg( \frac{\trace\bigl(\Sigma_V^{1/2}\bigr)}{K} \bigg)^2.
\]
In other words, the final expression of the optimizer $M^*$  in the problem \eqref{DL 2} is
\begin{equation}
	\label{DL solution}
	M^* = \frac{K}{\trace\bigl( \Sigma_V^{1/2} \bigr)}\Sigma_V^{1/2}.
\end{equation}
It follows that the optimal value of the problem \eqref{DL 2} (and therefore of \eqref{DL 1}) is $\dfrac{\bigl(\trace (\Sigma_V^{1/2})\bigr)^2}{K}$. Therefore, this value must be a lower bound of the optimal value, if it exists, for the problem \eqref{DL problem}.

Employing Corollary \ref{theorem 2} with $A = I_n$, we decompose $ M^* $ as 
\begin{equation}
	\label{e:Mstar decomposition}
	M^* = \sum_{i = 1}^K \dict{i}^*{\dict{i}^*}^\top\quad\text{with } \norm{\dict{i}^*} = 1 \text{ for each }i = 1, \ldots,K.
\end{equation}
Let us consider the dictionary $D_K^*$ consisting of the vectors $\{\dict{i}^*\}_{i = 1}^K$ obtained above. Since $X_V = \mathbb{R}^n$, the matrices $\Sigma_V, \Sigma_V^{1/2}$, and $M^*$ are of rank $n$, and therefore, $\Span\{\dict{i}^*\}_{i = 1}^K = \mathbb{R}^n $. Along with the fact that $\norm{\dict{i}^*} = 1$, we see that the dictionary $D_K^*$ of vectors $\{\dict{i}^*\}_{i = 1}^K$ is feasible for the problem \eqref{DL problem}.

Let us define the scheme 
\[
	\R^n\ni v\mapsto f^*_{D_K^*}(v) \Let \pmat{\dict{1}^*}{\dict{2}^*}{\dict{K}^*}^+ v \in\R^K.
\]
It is evident that this scheme $f^*_{D_K^*}$ is feasible for \eqref{DL problem}. But then the objective function in \eqref{DL problem} evaluated at $D_K = D_K^*$ and $f = f_{D_K^*}^*$ must be equal to $\dfrac{\bigl(\trace (\Sigma_V^{1/2})\bigr)^2}{K}$. Since this particular value is also a lower bound for the optimal value of \eqref{DL problem}, the problem \eqref{DL problem} is solvable. An optimal dictionary-scheme pair is given by
\begin{equation}
	\label{DL solution 1}
	\begin{cases}
		D_K^* = \{\dict{i}^*\}_{i = 1}^K \text{ obtained from the decomposition \eqref{e:Mstar decomposition}, and}\\
		\R^n\ni v\mapsto f^*(v) \Let \pmat{\dict{1}^*}{\dict{2}^*}{\dict{K}^*}^+v \in\R^K. 
	\end{cases}
\end{equation}
The proof is now complete.
\end{proof}

We provide the Algorithm \ref{procedure 2} that computes optimal dictionary-scheme pairs for the case $X_V = \mathbb{R}^n$. The inputs to the algorithm are the matrix $\Sigma_V$ and the size \(K\) of a dictionary:

\begin{algorithm}[H]
\caption{$\ell_2$-optimal dictionary for the case $X_V = \mathbb{R}^n$.}
\label{procedure 2}
\KwIn{A matrix $\Sigma_V\in\posdef{n}$ and a number $K \geq n$.}
\KwOut{An $\ell_2$-optimal dictionary-scheme pair $\bigl(\{\dict{i}^*\}_{i = 1}^K, f^*\bigr)$.}
\nl Define $M_1  \Let \frac{K}{\trace\bigl( \Sigma_V^{1/2} \bigr)}\; \Sigma_V^{1/2}$.

\nl Define $ M_2 \Let \begin{pmatrix}
	M_1 & O_{n\times (K-n)} \\
	O_{(K-n)\times n} & I_{K-n}
\end{pmatrix}
$ , 
$
A \Let \begin{pmatrix}
	I_n & O_{n\times(K-n)} \\
	O_{(K-n)\times n} & O_{(K-n)\times(K-n)}
\end{pmatrix}
$

\nl Compute $C \in \mathbb{R}^{K \times K}$ such that $M_2 = CC^\top $.

\nl Define $\Lambda \in \mathbb{R}^{K \times K}$ by $\Lambda \Let C^\top A C$, and apply Algorithm \ref{Algo1} to get a collection of vectors $\{x_i\}_{i = 1}^K \subset \mathbb{R}^K$.

\nl Define the collection $\{v_i\}_{i = 1}^K \subset \mathbb{R}^K$ by $v_i \Let C x_i$ for $i = 1,\ldots, K $.

\nl Define the $\ell_2$-optimal dictionary $\{\dict{i}^*\}_{i = 1}^K \subset \mathbb{R}^n$ such that the $j^{th}$ component of $\dict{i}^*$ is given by $\dict{i}^*(j) \Let v_i(j)$ for $j = 1,\ldots,n$ and for every $ i =  1,\ldots,K$.  

\nl Define the optimal scheme $\R^n\ni v\mapsto f^*(v) \Let  \pmat{\dict{1}^*}{\dict{2}^*}{\dict{K}^*}^+ v$.
\end{algorithm}


\subsection{Proof of Lemma \ref{lemma 2}}

\begin{proof}
We argue by contradiction. Suppose that the assertion of the Lemma is false. If we denote by $x_i$ the orthogonal projection of $\dict{i}$ on $X_V$ and by $y_i$ the orthogonal projection of $\dict{i}$ on the orthogonal complement of $X_V$, we must have $\norm{x_i} < 1$ for at least one value of $i$. If $f$ is an optimal scheme of representation, feasibility of $f$ gives, for any $v \in R_V$,
\begin{equation}\label{lemma 2 eq 1}
\begin{aligned}
	v	& = \sum_{i = 1}^K \dict{i} f_i(v) = \biggl(\sum_{i = 1}^K x_i f_i(v)\biggr) + \biggl(\sum_{i = 1}^K y_i f_i(v)\biggr)\\
		 & = \sum_{\substack{i = 1,\\ \norm{x_i} \neq 0}}^K x_i f_i(v) + 0.
\end{aligned}
\end{equation}
Fix a unit vector \(x\in X_V\), and define a dictionary $\{\dict{i}^*\}_{k=1}^K$ by
\begin{align*}
\dict{i}^* \Let 
\begin{dcases}
\frac{x_i}{\norm{x_i}} & \text{if } \norm{x_i} \neq 0,\\
x	& \text{otherwise}.
\end{dcases}
\end{align*}
Then clearly 
\[
	\Span\{\dict{i}^*\}_{i = 1}^K \supset \Span\{x_i\}_{i = 1}^K \supset R_V \quad\text{and}\quad \norm{\dict{i}^*} = 1 \text{ for all }i = 1,\ldots,K.
\]
In other words, the dictionary of vectors $\{\dict{i}^*\}_{i = 1}^K$ is feasible for the problem \eqref{DL problem general}. Let us now define a scheme $f^*$ by
\[
	\R^n\ni v\mapsto f^*(v) \Let \diag\{\norm{x_1},\norm{x_2},\ldots, \norm{x_K}\}f(v)\in\R^K.
\] 
For any $v \in R_V$, using the dictionary consisting of vectors $\{\dict{i}^*\}_{i = 1}^K$ we get
\begin{equation}\label{lemma 2 eq 2}
\begin{aligned}
	\sum_{i = 1}^K \dict{i}^* f^*_i(v) &= \sum_{i = 1}^K \dict{i}^* \norm{x_i}f_i(v) = \sum_{\substack{i = 1,\\\norm{x_i} \neq 0}}^K \frac{x_i}{\norm{x_i}} \norm{x_i}f_i(v) = v,
\end{aligned}
\end{equation}
where the last equality follows from \eqref{lemma 2 eq 1}. Thus, $f^*(\cdot)$ along with the dictionary of vectors $\{\dict{i}^*\}_{i = 1}^K$ is feasible for problem \eqref{DL problem general}. But for any $v \in R_V$ we have
\[
\norm{f^*(v)}^2 = \sum_{i = 1}^K \bigl(f^*_i(v)\bigr)^2 = \sum_{i = 1}^K \norm{x_i}^2 \bigl(f_i(v)\bigr)^2 < \sum_{i = 1}^K \bigl(f_i(v)\bigr)^2 = \norm{f(v)}^2,
\]
where the inequality is due to the fact that $\norm{x_i} < 1$ for at least one $i$. This contradicts the assumption that the pair $\{\dict{i}\}_{i = 1}^K$ along with the scheme $f$ is optimal for \eqref{DL problem general}.
\end{proof}

\subsection{Proof of Theorem \ref{DL general theorem}}
\label{sec:proof-of-general-theorem}

\begin{proof}
The problem \eqref{DL problem general 2} is similar to problem \eqref{DL problem} except for the first constraint. In \eqref{DL problem} we optimize over vectors taking values on the surface of the unit sphere, whereas in \eqref{DL problem general 2} we optimize over vectors taking values on the surface of the ellipsoid $\set[\big]{x \in \mathbb{R}^m \suchthat x^\top (B^\top B)x = 1}$. Following the arguments in the proof of Theorem \ref{DL theorem} till \eqref{DL 1}, one can conclude that the optimal value, if it exists, of problem \eqref{DL problem general 2} is bounded below by the optimal value, if it exists, of the problem
\begin{equation}\label{modified genral DL}
\begin{aligned}
	& \minimize_{\{\dictx{i}\}_{i=1}^K}	&& \trace \biggl( \Sigma_{V_X} \biggl( \sum_{i = 1}^K\dictx{i}\dictx{i}^\top \biggr)^{-1}\biggr)\\
	& \sbjto	&&
	\begin{cases}
		\dictx{i}^\top (B^\top B) \dictx{i}  = 1\text{ for all } i = 1,2, \ldots, K, \\
		\Span\{\dictx{i}\}_{i = 1}^K = \mathbb{R}^m,
	\end{cases}
\end{aligned}
\end{equation}
where $\Sigma_{V_X} \Let \EE_\mu \bigl[V_XV^\top_X\bigr] = \bigl((B^\top B)^{-1}B^\top\bigr) \EE_{\mu}\bigl[ VV^\top \bigr] \bigl((B^\top B)^{-1}B^\top \bigr)^\top $.

Let us define:
\begin{itemize}[label=\(\circ\), leftmargin=*]
	\item \(S\) to be the feasible region of the problem \eqref{modified genral DL},
	\item $R \Let \set[\big]{H \in \posdef{m} \suchthat \trace\bigl(H(B^\top B)\bigr)) = K}$, and
	\item the map $\bigl(\mathbb{R}^{m}\bigr)^K \ni (\dictx{1}, \dictx{2}, \ldots, \dictx{K})\mapsto  h(\dictx{1}, \dictx{2}, \ldots, \dictx{K}) \Let \sum_{i = 1}^K \dictx{i}\dictx{i}^\top \in \possemdef{m} $.
\end{itemize}
From Corollary \ref{theorem 2} we see that for every $H \in R$ there exists a collection of vectors $\{\dictx{i}\}_{i = 1}^K$ such that 
\[
	\sum_{i = 1}^K \dictx{i}\dictx{i}^\top = H \quad\text{and}\quad \dictx{i}^\top (B^\top B)\dictx{i} = \frac{\trace\bigl(H(B^\top B)\bigr)}{K} = 1,
\]
which, along with the fact that $\rank(H) = m \Rightarrow \Span\{\dictx{i}\}_{i = 1}^K = \mathbb{R}^m$, imply that 
\begin{equation}\label{Theorem 4 eq 1}
R \subset h(S).
\end{equation}
Moreover, for any collection $\{\dictx{i}\}_{i = 1}^K \in S$, we have
\[
	\trace\bigl( h(\dictx{1}, \dictx{2}, \ldots, \dictx{K})(B^\top B) \bigr) = \sum_{i = 1}^K \dictx{i}^\top (B^\top B) \dictx{i} = K\quad\text{and}\quad h(\dictx{1}, \dictx{2}, \ldots, \dictx{K}) \in \posdef{m},
\]
which implies that
\begin{equation}\label{Theorem 4 eq 2}
h(S) \subset R.
\end{equation} 
From \eqref{Theorem 4 eq 1} and \eqref{Theorem 4 eq 2} we conclude that $R = h(S)$. In other words, instead of optimizing over the feasible collection of vectors in $S$ in \eqref{modified genral DL}, one can equivalently optimize over the set of symmetric positive definite matrices in $R$. This consideration leads us to the problem:
\begin{equation}\label{modified general DL 2}
\begin{aligned}
	& \minimize_{H \; \in \; \posdef{m}}	&& \trace \bigl(\Sigma_{V_X}H^{-1}\big)\\
	& \sbjto	&& \trace\bigl(H(B^\top B)\bigr) - K = 0.
\end{aligned}
\end{equation}

Letting $M \Let (B^\top B)^{1/2}H(B^\top B)^{1/2}$, we write the optimization problem \eqref{modified general DL 2} with $M$ as the variable instead of $H$. Due to this change of variables, the constraint and the objective function become
\[
	\trace\bigl(H(B^\top B)\bigr) = \trace\bigl((B^\top B)^{1/2}H(B^\top B)^{1/2}\bigr) = \trace(M),
\]
and
\begin{equation}\label{Theorem 4 eq 3}
\begin{aligned}
\trace\bigl(\Sigma_{V_X}H^{-1}\bigr) &= \trace\bigl(\Sigma_{V_X}{(B^\top B)^{1/2}M^{-1}(B^\top B)^{1/2}}\bigr) \\
&= \trace\bigl((B^\top B)^{1/2}\Sigma_{V_X}(B^\top B)^{1/2}M^{-1}\bigr)\\
&= \trace\bigl(\Sigma M^{-1}\bigr),
\end{aligned}
\end{equation}
where 
\begin{equation}
\begin{aligned}
\Sigma &\Let (B^\top B)^{1/2}\Sigma_{V_X}(B^\top B)^{1/2} \\
	   &= (B^\top B)^{1/2} \bigl((B^\top B)^{-1}B^\top\bigr) \EE_{\mu}\bigl[ VV^\top \bigr] \bigl((B^\top B)^{-1}B^\top\bigr)^\top (B^\top B)^{1/2}\\
&= (B^\top B)^{-1/2} \bigl(B^\top \Sigma_V B \bigr) (B^\top B)^{-1/2}.
\end{aligned}
\end{equation}
Using \eqref{Theorem 4 eq 3} we write the problem \eqref{modified general DL 2} equivalently as:
\begin{equation}\label{modified general DL 3}
\begin{aligned}
	& \minimize_{M \in \posdef{n}}	&&  \trace\bigl(\Sigma M^{-1}\bigr)\\
	& \sbjto	&& \trace(M) - K = 0.
\end{aligned}
\end{equation}
The problem \eqref{modified general DL 3} is identical to \eqref{DL 2}, which implies that the problem \eqref{modified general DL 3} is solvable, and an optimizer is
\[
	M^* \Let \; \frac{K}{\trace\bigl(\Sigma^{1/2}\bigr)} \Sigma^{1/2}.
\]
Therefore, the problem \eqref{modified general DL 2} is solvable, and an optimizer is 
\begin{equation}
\label{H1}
\begin{aligned}
H^*	& \Let (B^\top B)^{-1/2}M^*(B^\top B)^{-1/2}\\
	& = \frac{K}{\trace\bigl(\Sigma^{1/2}\bigr)}\bigl((B^\top B)^{-1/2}\Sigma^{1/2}(B^\top B)^{-1/2}\bigr).
\end{aligned}
\end{equation}
From Corollary \ref{theorem 2} it follows that there exists a collection $\{\dictx{i}^*\}_{i = 1}^K$ of vectors such that 
\[
	\sum_{i = 1}^K \dictx{i}^* {\dictx{i}^*}^\top = H^* \quad \text{and}\quad {\dictx{i}^*}^\top \bigl(B^\top B\bigr)\dictx{i}^* = \frac{\trace\bigl(H^*(B^\top B)\bigr)}{K} = 1.
\]

Employing arguments similar to those given in the proof of Theorem \ref{DL theorem}, we now conclude that the pair
\begin{itemize}[label=\(\circ\), leftmargin=*]
	\item the collection of vectors $\{\dictx{i}^*\}_{i = 1}^K$, and
	\item the scheme $f_X^*(u) = \pmat{\dictx{1}^*}{\dictx{2}^*}{\dictx{K}^*}^+ u$,
\end{itemize}
is optimal for the problem \eqref{DL problem general 2}. Using the optimal solution of \eqref{DL problem general 2}, we define a dictionary-scheme pair as:
\begin{equation}
\label{DL general solution}
\begin{cases}
\dict{i}^* \Let B \dictx{i}^* \quad \text{for $i = 1,\ldots,K$},\\
\R^n\ni v\mapsto f^*(v) \Let f_X^*\Bigl(\bigl((B^\top B)^{-1}B^\top \bigr)v \Bigr) = \pmat{\dictx{1}^*}{\dictx{2}^*}{\dictx{K}^*}^+ \bigl((B^\top B)^{-1}B^\top \bigr) v.
\end{cases}
\end{equation}
It is clear that the pair in \eqref{DL general solution} is feasible for the problem \eqref{DL problem general}, and that the corresponding objective function evaluates to the optimal value of the problem \eqref{modified genral DL}. Therefore, along with the assertion of Lemma \ref{lemma 2} we can conclude that the problem \eqref{DL problem general} is solvable, and in fact an optimal solution is given by \eqref{DL general solution} with the optimal value of $\dfrac{\left(\trace (\Sigma^{1/2})\right)^2}{K}$. This completes the proof.
\end{proof}

As in the case $X_V = \mathbb{R}^n$, we now provide the Algorithm \ref{algo:general_l2_opt_algo} to obtain an optimal dictionary-scheme pair for the general $\ell_2$-optimal dictionary problem \eqref{DL problem general}. The algorithm takes the matrix $\Sigma_V$ and the size of the dictionary \(K\)  as its inputs. From \(\Sigma_V\) we extract a matrix $B \in \mathbb{R}^{n \times m}$ containing a set of basis vectors for $\image (\Sigma_V)$ in its columns, these vectors form a basis for $X_V$.

\begin{algorithm}
\caption{A procedure to obtain $\ell_2$-optimal dictionary.}
\label{algo:general_l2_opt_algo}
\KwIn{A matrix $\Sigma_V\in \possemdef{n}$ and a number $K \geq m \Let \dim(X_V) = \rank(\Sigma_V)$.}
\KwOut{An $\ell_2$-optimal dictionary-scheme pair $\bigl(\{y_i^*\}_{i = 1}^K, f^*\bigr)$.}

\nl Compute a basis $\{b_i\}_{i = 1}^m$ for $\image(\Sigma_V)$ and define $B \Let \pmat{b_1}{b_2}{b_m}$.

\nl Define $\Sigma \Let (B^\top B)^{-1/2} \bigl(B^\top \Sigma_V B\bigr) (B^\top B)^{-1/2}$.

\nl Compute $H \Let \frac{K}{\trace\bigl(\Sigma^{1/2}\bigr)}\bigl((B^\top B)^{-1/2}\Sigma^{1/2}(B^\top B)^{-1/2}\bigr)$.

\nl Define
$
M \Let \begin{pmatrix}
	H & O_{m\times (K-m)} \\
	O_{(K-m)\times m} & I_{K-m}
\end{pmatrix}
$ , 
$
A \Let \begin{pmatrix}
	B^\top B & O_{m\times(K-m)} \\
	O_{(K-m)\times m} & O_{(K-m)\times(K-m)}
\end{pmatrix}
$

\nl Compute $C \in \mathbb{R}^{K \times K}$ such that $M = CC^\top $.

\nl Define $\Lambda \in \mathbb{R}^{K \times K}$ by $\Lambda \Let C^\top A C$, and apply Algorithm \ref{Algo1} to get a collection of vectors $\{x_i\}_{i = 1}^K \subset \mathbb{R}^K$.

\nl Define the collection $\{v_i\}_{i = 1}^K \subset \mathbb{R}^K$ as $v_i \Let C x_i$ for $i = 1,\ldots, K $.

\nl  Define the collection $\{\dictx{i}^*\}_{i = 1}^K \subset \mathbb{R}^m$ such that the $j^{th}$ component of $\dictx{i}^*$ is given by $\dictx{i}^*(j) \Let v_i(j)$ for $j = 1,\ldots,m$ and for every $ i =  1,\ldots,K$.  

\nl Define the $\ell_2$-optimal dictionary $\{\dict{i}^*\}_{i = 1}^K \subset \mathbb{R}^n$ as $\dict{i}^* \Let B \dictx{i}^*$ for $ i = 1, \ldots,K $.

\nl Define the optimal scheme $\mathbb{R}^n \ni v\mapsto f^*(v) \Let \pmat{\dict{1}^*}{\dict{2}^*}{\dict{K}^*}^+ v\in\R^K$.

\end{algorithm}

\section{Conclusion and future directions}
\label{s:conclusion}
In this article we have provided an explicit solution of the $\ell_2$-optimal dictionary problem in the form of a \emph{rank-\(1\) decomposition} of a specific positive definite matrix derived from given data, together with algorithms to compute the corresponding $\ell_2$-optimal dictionaries. 

The analysis in this article assumes that the second moment of the random vector whose samples are to be represented is known. An online algorithm which estimates the second moment of the random vector and computes the dictionary vectors in parallel is being developed, and will be reported in subsequent articles. 

\subsection*{Acknowledgements}
We sincerely thank Prof.\ V.\ S.\ Borkar for his valuable suggestions towards convexity of the $\ell_2$-optimal dictionary problem, Prof.\ K.\ S.\ Mallikarjuna Rao for pointing us to the literature on rank-\(1\) decompositions of matrices, and Prof.\ N.\ Khaneja for helpful discussions.

\vskip 0.2in
\bibliographystyle{alpha}
\bibliography{ref}

\end{document}